\documentclass[11pt]{article}
\usepackage{fullpage}

\usepackage{authblk}

\date{}

\title{\bfseries\papertitle}

\author[1]{Aadirupa Saha$^*$}
\author[2]{Shubham Gupta$^*$}

\affil[1]{Microsoft Research, New York City, US}
\affil[2]{IBM Research, Orsay, France.}

\usepackage[T1]{fontenc}
\usepackage{lmodern}
\usepackage[utf8]{inputenc} 
\usepackage{url}            
\usepackage{booktabs}       
\usepackage{nicefrac}       
\usepackage{microtype}      
\usepackage{xcolor}         

\usepackage[square,numbers,sort&compress]{natbib}
\usepackage{nameref}
\usepackage[colorlinks, linkcolor=blue, filecolor = blue, citecolor = blue, urlcolor = blue]{hyperref}
\usepackage{amsmath, amsthm, amssymb, amsfonts}
\usepackage{bm}
\usepackage{hyperref}
\usepackage{algorithm, algorithmic}
\usepackage{tabularx}
\usepackage{thm-restate}
\usepackage{thmtools}
\usepackage{subfig}
\usepackage{color}
\usepackage{enumitem}
\usepackage[capitalize,noabbrev]{cleveref}

\usepackage{mathtools}

\theoremstyle{plain}
\newtheorem{theorem}{Theorem}[section]

\newtheorem{lemma}[theorem]{Lemma}

\theoremstyle{definition}

\theoremstyle{remark}
\newtheorem{remark}[theorem]{Remark}

\newcommand{\N}{{\mathbb N}}
\renewcommand{\P}{{\mathbf P}}

\newcommand{\1}{{\mathbf 1}}

\newcommand{\cA}{{\mathcal A}}

\newcommand{\cI}{{\mathcal I}}

\newcommand{\cT}{{\mathcal T}}

\newcommand{\tO}{\tilde{O}}

\newcommand{\nstdb}{{\texttt{NSt-DB}}}

\newcommand{\sv}{\textrm{Switching-Variation}}
\newcommand{\cv}{\textrm{Continuous-Variation}}

\newcommand{\p}{{\mathbf p}}

\newcommand{\sm}{\setminus}

\DeclareMathOperator*{\argmax}{argmax}

\def \papertitle{{Optimal and Efficient Dynamic Regret Algorithms for \\ Non-Stationary Dueling Bandits}}

\newcommand{\sreg}{\textrm{SR}}
\newcommand{\dreg}{\textrm{DR}}

\newcommand{\dbreg}{\textrm{DBR}}

\newcommand{\algp}{\texttt{DEX3.P}}
\newcommand{\algs}{\texttt{DEX3.S}}
\newcommand{\algdexp}{\texttt{Borda-DEX3.S}}
\newcommand{\rand}{\texttt{RAND}}
\newcommand{\rex}{\texttt{REX3}}

\makeatletter
\newcommand{\multiline}[1]{%
  \begin{tabularx}{\dimexpr\linewidth-\ALG@thistlm}[t]{@{}X@{}}
    #1
  \end{tabularx}
}
\makeatother

\newcommand{\bfone}[1]{\mathbf{1}\left\lbrace #1 \right\rbrace}

\newcommand{\bfp}{\mathbf{p}}

\newcommand{\bfP}{\mathbf{P}}

\newcommand{\rmE}[1]{\mathrm{E}\left[ #1 \right]}

\newcommand{\rmP}[1]{\mathrm{P}\left( #1 \right)}

\newcommand{\calE}{\mathcal{E}}
\newcommand{\calF}{\mathcal{F}}

\newcommand{\calH}{\mathcal{H}}

\newcommand{\calN}{\mathcal{N}}

\newcommand{\calV}{\mathcal{V}}

\newcommand{\bbI}{\mathbb{I}}
\newcommand{\bbN}{\mathbb{N}}
\newcommand{\bbR}{\mathbb{R}}

\newcommand{\bmpi}{\bm{\pi}}
\newcommand{\bmmu}{\bm{\mu}}

\newcommand{\abs}[1]{\left\vert #1 \right\vert}

\newcommand{\rmTV}[2]{\mathrm{TV}\left( #1, #2 \right)}
\newcommand{\rmKL}[2]{\mathrm{KL}\left( #1, #2 \right)}

\begin{document}

\maketitle

\def\thefootnote{*}\footnotetext{Equal contribution.}

\begin{abstract}
We study the problem of \emph{dynamic regret minimization} in $K$-armed Dueling Bandits under non-stationary or time-varying preferences. This is an online learning setup where the agent chooses a pair of items at each round and observes only a relative binary `win-loss' feedback for this pair sampled from an underlying preference matrix at that round. 
We first study the problem of static-regret minimization for adversarial preference sequences and design an efficient algorithm with $\tO(\sqrt{KT})$ regret bound.
We next use similar algorithmic ideas to propose an efficient and provably optimal algorithm for dynamic-regret minimization under two notions of non-stationarities.
In particular, we show $\tO(\sqrt{SKT})$ and $\tO({V_T^{1/3}K^{1/3}T^{2/3}})$ dynamic-regret guarantees, respectively, with $S$ being the total number of `effective-switches' in the underlying preference relations and $V_T$ being a measure of `continuous-variation' non-stationarity. These rates are provably optimal as justified with matching lower bound guarantees.
Moreover, our proposed algorithms are flexible as they can be easily `blackboxed'  to yield dynamic regret guarantees for other notions of dueling bandits regret, including condorcet regret, best-response bounds, and Borda regret.
The complexity of these problems have not been studied prior to this work despite the practicality of non-stationary environments.
Extensive simulations corroborate our results.
\end{abstract}

\vspace{-20pt}

\section{Introduction}
\label{sec:intro}

The problem of \emph{Dueling-Bandits} has gained much attention in the machine learning community \citep{Yue+12,Zoghi+14RCS,Zoghi+15}. This is an online learning framework that generalizes the standard multi-armed bandit (MAB) setting  \citep{Auer+02,TS12,CsabaNotes18,Audibert+10,Kalyanakrishnan+12} by querying relative preference feedback of actively chosen item pairs instead of an absolute feedback for a single item. More formally, in dueling bandits, the learning proceeds in rounds: At each round, the learner selects a pair of arms and observes stochastic preference feedback of the winner of a comparison (duel) between the selected arms. The objective of the learner is to minimize the regret with respect to a (or a set of) `best' arm(s) in hindsight.   Towards this, several algorithms have been proposed \citep{Ailon+14,Zoghi+14RUCB,Komiyama+15,Adv_DB}. Due to the inherent exploration-vs-exploitation tradeoff of the learning framework and several advantages of preference feedback \citep{Busa14survey,Yue+09}, many real-world applications can be modeled as dueling bandits, including movie recommendations, retail management, search engine optimization, and job scheduling, etc.

While most existing works have studied the dueling bandit problem in a stochastic setting where the underlying preferences are assumed to be fixed over time, it is often unjustified in reality. Specifically, depending on the time, season, demographics, occasions, etc., the underlying preferences may often change over time. E.g., a recommendation system must consider that a user's preferred list of items (videos) depends on factors such as the time of the day and seasonal needs. Similarly, the relevance of a search query changes with the location and market influence.
It is therefore more reasonable to assume  adversarial feedback models for dueling bandits where, unlike the stochastic setting, the underlying preference matrices can vary over time. This naturally leads to the following question.

\textbf{Q.1:} \emph{Can we design an efficient algorithm that is competitive under adversarial changes in the preference model?}

Unfortunately, except for some handful of attempts \cite{Ailon+14, Adv_DB}, which are restricted to a very special class of linear-score based preference structures (see details in Sec. \ref{sec:rel}), none of the existing dueling bandits algorithms can guarantee an order optimal $O(\sqrt{T})$ regret rate against any fixed benchmark (competitor) for any arbitrary adversarial sequence of preference matrices.

More importantly, drawing inspiration from the necessity of analyzing non-stationary preference models, a more ambitious and practical regret objective is to evaluate the performance of learning algorithms in terms of how quickly they adapt to a new environment, i.e., where they are actually compared against a dynamic benchmark based on the underlying preferences instead of a single fixed benchmark considered above (see Q.1). This problem is often studied as dynamic regret minimization in standard multiarmed bandits (MAB) literature \citep{besbes+14,besbes+15,auer+02N,luo+18,luo+19,wei21,zhu1,zhu2}, where the learner gets to observe the absolute reward upon pulling an arm. The next, and perhaps a daunting, question to ask is

\textbf{Q.2:} \emph{How can we design efficient and optimal dynamic regret algorithms for the dueling bandits setup?}

Surprisingly, despite dueling bandits being a significantly well studied extension of MAB, existing literature still lacks efficient algorithms for dynamic regret minimization in this framework. To the best of our knowledge, \emph{we are the first to extend the stochastic $K$-armed dueling bandit problem to non-stationary preferences} (a.k.a. \nstdb).
\begin{figure}[h]
	\begin{center}
	\resizebox{0.6\linewidth}{!}{
		\begin{tabular}{l}
			\toprule
			\textbf{Parameters. }Item set: $[K]$ \\
			\hline
			\textbf{For} $t = 1, 2, \ldots, T$, the learner: \\
		~ Chooses $(k_{+1, t}, k_{-1, t}) \in [K]\times [K]$  \\
		~ Observes $o_t := \bbI\{k_{+1, t} \succ k_{-1, t}\} \sim \text{Ber}(P_t(k_{+1, t}, k_{-1, t}))$  \\
		~ Incurs $r_t := \frac{1}{2} \left(P_t(i_t^*, k_{+1, t}) + P_t(i_t^*, k_{-1, t}) - 1\right)$;
	$i_t^* \in [K]$\\
			\bottomrule
		\end{tabular}}
	\end{center}
	\vspace{-10pt}
  \caption{Setting of Non-Stationary DB (\nstdb)} 
  \label{fig:setting}
\end{figure}

\textbf{Problem definition (informal).}
Let $[K]$ be the item (action) set and $\P_t \in [0, 1]^{K \times K}$ denote the underlying preference matrix at time $t$. The learner is allowed to choose two actions $\{k_{+1, t}, k_{-1, t}\} \subseteq [K]$ at time $t$, upon which it observes an outcome $o_t \sim \text{Ber}(P_t(k_{+1, t}, k_{-1, t}))$ and incurs a regret $r_t := \frac{1}{2} \left(P_t(i_t^*, k_{+1, t}) + P_t(i_t^*, k_{-1, t}) - 1\right)$ w.r.t. any `good arm' $i_t^* \in [K]$ (see Fig. \ref{fig:setting}). The goal is to provide optimal dynamic-regret algorithms under different `measures of non-stationarity' (see details in Sec. \ref{sec:prob}).

\textbf{Contributions.} Our specific contributions are:

\begin{itemize}[nosep]

\item As motivated in Q.1, we first study the `static dueling bandit regret' (see Sec. \ref{sec:prob} for definition) w.r.t. any fixed arm $i^* \in [K]$ (or more generally any distribution $\bmpi^* \in \Delta_K$) for adversarially chosen preferences. Our algorithm (Alg. \ref{alg:3p}) guarantees an $\tO(\sqrt{KT})$ high probability regret bound (Thm. \ref{theorem:3p_bound}, Rem. \ref{rem:exp3p}, Sec. \ref{section:3p}). 

This is one of our primary contributions that shows how  techniques from standard EXP3 algorithm \citep{auer+02N} can be borrowed to establish high-probability $\tO(\sqrt{KT})$ regret bounds for dueling bandits against any fixed benchmark. Sec. \ref{section:3p_analysis} details the main ideas behind the analysis. The result is not only useful on its own but also serves as the basic building block for our subsequent non-stationary algorithms.

\item Towards addressing the \nstdb\, problem (see Q.2), our second (and main) contribution lies in analyzing \emph{Non-Stationary Dueling Bandits}. We design an algorithm (Alg. \ref{alg:3s}) with high probability dynamic-regret guarantees for two measures of non-stationarities: `\sv' $S$ and `\cv' $V_T$ (Definitions in Sec. \ref{sec:prob}). More precisely, the upper bound on the dynamic regret of our proposed algorithm are shown to be $O(\sqrt{SKT})$ and $O({V_T^{1/3}K^{1/3}T^{2/3}})$, respectively (see Thm. \ref{theorem:3s_bound}, \ref{theorem:3s_boundcv}, Sec. \ref{section:3s}).

\item The regret optimality of our algorithms is justified by proving matching lower bound guarantees for dynamic-regret under both the notions of non-stationarities (Thm. \ref{theorem:lower_bound_switch}, \ref{theorem:lower_bound_continuous}, Sec. \ref{section:lower_bound}).

\item We further show the versatility of the proposed algorithm, which can be modularized and applied as a blackbox recipe to more general forms of dynamic regrets, including one that we analyzed above as well as Borda score based dynamic regret (Sec. \ref{sec:ub_borda}).

\item Finally, our theoretical analysis is corroborated with detailed empirical evaluations to compare the performance of the proposed techniques over state-of-the-art dueling bandit algorithms (Sec. \ref{section:experiments}).
\end{itemize}

\subsection{Related Works}
\label{sec:rel}

The problem of regret minimization in stochastic multi-armed bandits (MAB) is very well-studied in the online learning literature \citep{auer+02N,CsabaNotes18,Audibert+10} where the learner pulls a single arm per round and sees a noisy sample of the absolute reward feedback for this arm. A well-motivated generalization of MAB framework is extending it to non-stationary environments where the underlying reward model changes with time \citep{besbes+15,auer+02N} and the goal of the learner is to minimize the \emph{dynamic regret} w.r.t. the `best-performing' benchmark (arm) at each round. This problem is well-studied in the online learning community for the standard $K$-item MAB \citep{besbes+14,besbes+15,auer+02N}, and the non-stationary contextual bandits \citep{luo+19,wei21} under different non-stationarities \citep{luo+18}.

On the other hand, the relative feedback variants of the stochastic MAB problem have seen a widespread resurgence in the form of Dueling Bandits over the last two decades. Several algorithms have been proposed to address regret minimization in dueling bandits for different notions of `best performing' arms depending on the underlying preference models, e.g. Condorcet, Borda, and Copeland winner \citep{Yue+12,Busa_pl,Zoghi+14RCS,Zoghi+15,DTS,CDB,SG19,SG20}. Recent works have also extended the pairwise preference feedback model to include subset-wise preferences \citep{Sui+17,Brost+16,SGrank18,SGwin18,Ren+18}.

\citet{Ailon+14} first studied the dueling bandit problem for adversarial setup and introduced the famous sparring EXP3 idea (without regret guarantees) which was later used in many follow-up works for proving regret guarantees on dueling bandits in adversarial environments \cite{Adv_DB,ADB,Sui+17}. However, unlike us, \emph{they do not consider the dynamic regret objective} w.r.t. a time-varying benchmark $i_t^*$.  
The settings in \citet{Ailon+14} and \cite{Adv_DB} are restricted to only utility based preferences where each item has a scalar utility at each time. This entails a complete ordering between the items at each step, which only covers a small subclass of the general $[K]\times [K]$ preference matrices. Additionally, \citet{Adv_DB} assume that the feedback includes not only the winner but also the difference in the utilities between the winning and losing item, which is more similar to MAB feedback and far from our $0/1$ one bit preference feedback. \citet{ADB} did consider the dueling bandit setup for general adversarial preferences, but they measure (static) regret in terms of \emph{Borda-scores} for which the fundamental regret performance limit is shown to be $\Omega(K^{1/3}T^{2/3})$. This measure of regret is very different from our preference-based regret objective (see Sec. \ref{sec:prob}). \citet{dudik+15} adapted sparring-EXP3 to the contextual dueling bandit setup where the learner competes against the best policy from a given policy class $\Pi$. Their proposed sparring-EXP4 algorithm is shown to give a $\tO(\sqrt{KT \log |\Pi|})$ rate, however the algorithm is computationally intractable for arbitrary context sets.
\emph{To the best of our knowledge, we are the first to design efficient and optimal dynamic regret algorithms for `non-stationary dueling bandits'.}

\section{Problem Formulation}
\label{sec:prob}

\textbf{Notations. } Let $[n] := \{1,2, \ldots n\}$ for any $n \in \N$. 
Given a set $\cA$, for any $a,b \in \cA$, we denote by $a \succ b$ the event that $a$ is preferred over $b$. $\Delta_{K}:= \{\p \in [0,1]^K \mid \sum_{i = 1}^K p_i = 1, p_{i} \ge 0, \forall i \in [K]\}$ denotes the $K$-simplex.

\textbf{Setting.}     
We assume a decision space of $K$ arms denoted by $\cA:= [K]$. At each round $t$, the task of the learner is to select a pair of actions $(k_{+1, t}, k_{-1, t}) \in [K]\times [K]$, upon which a preference feedback $o_t \sim \text{Ber}(P_t(k_{+1, t}, k_{-1, t}))$ is revealed to the learner according to the underlying preference matrix $\P_t \in [0,1]^{K \times K}$ (chosen adversarially), such that the probability of $k_{+1, t}$ being preferred over $k_{-1, t}$ is given by $\rmP{o_t = 1} \coloneqq \rmP{k_{+1, t} \succ k_{-1, t}} = P_t(k_{+1, t},k_{-1, t})$, and hence $\rmP{o_t = 0} \coloneqq \rmP{k_{-1, t} \succ k_{+1, t}} = 1 - P_t(k_{+1, t},k_{-1, t})$.

\textbf{Measures of Non-Stationarities.} We measure the dynamic regret w.r.t. the following notions of non-stationarity:

\begin{enumerate}
\item \textbf{\sv} $S:= \sum_{t=2}^T \1[\P_t \not= \P_{t-1}]$.
\item \textbf{\cv}: $$V_T:= \sum_{t=2}^T \max_{(a,b) \in [K]\times [K]}|P_t(a,b) - P_{t-1}(a,b)|.$$
\end{enumerate}

\textbf{Objective (Static and Dynamic Regret).} Assuming the learner selects the duel $(k_{+1, t},k_{-1, t})$ at round $t$, one can measure its performance w.r.t. a single fixed arm $i^* \in [K]$\footnote{Note that this is equivalent to maximizing the expected regret w.r.t. any fixed distribution $\bmpi^* \in \Delta_K$, i.e. when $i^* \sim \bmpi^*$. This is because the regret objective is linear in the entries of $\bmpi^*$, so the maximizer $\bmpi^*$ is always one hot.} in hindsight by calculating the Static Regret:

\small
\begin{align}
	\label{eq:sreg}
	\sreg_T := \max_{i^* \in [K]} \sum_{t=1}^T \frac{1}{2} \left( P_t(i^*, k_{+1, t}) + P_t(i^*, k_{-1, t}) -1 \right). 
\end{align}
\normalsize

We further study a \emph{stronger notion} of \emph{dynamic regret} with respect to a time varying benchmark sequence $j^T = (j_1, j_2, \dots, j_T)$, defined as:

\small
\begin{align}
	\label{eq:dreg}
	\dreg_T(j^T) := \sum_{t=1}^T \frac{1}{2}\left( P_t(j_t, k_{+1, t}) + P_t(j_t, k_{-1, t}) -1 \right). 
\end{align}
\normalsize
Let $r_t(j) = \frac{1}{2} (P_t(j, k_{+1, t}) + P_t(j, k_{-1, t}) -1)$. The learner competes against the `best-response' of $(k_{+1, t}, k_{-1, t})$ at time $t$ if $j_t = j_t^* = \argmax_{j \in [K]} r_t(j)$. On the other hand, it is easy to see that if we set $j_t^* = j^* = \argmax_{j \in [K]} \sum_{t = 1}^T r_t(j)$, we have $\dreg_T(j^T) = 	\sreg_T$. We now proceed to describe our algorithms for Static and Dynamic regret minimization.

\begin{algorithm}
  \caption{\algp{}}
  \label{alg:3p}
  \begin{algorithmic}[1]
    \STATE \textbf{Input:} $\eta > 0$, $\gamma, \beta \in (0, 1)$
    \STATE Set $W_{+1, 1}(k) = W_{-1, 1}(k) = 1$ for all $k \in [K]$
    \FOR {$t = 1, \dots, T$}
      \FOR {$k \in [K]$ and $i \in \{+1, -1\}$}
        \STATE $p_{i, t}(k) = (1 - \gamma) \frac{W_{i, t}(k)}{\sum_{j=1}^K W_{i, t}(j)} + \frac{\gamma}{K}$
      \ENDFOR
      \STATE For $i \in \{+1, -1\}$, sample $k_{i, t}$ from the distribution $(p_{i, t}(1), \dots, p_{i, t}(K))$
      \STATE Observe preference feedback $o_t(k_{+1, t}, k_{-1, t})$
      \STATE Compute $\hat{g}_{i, t}(k)$ for $i \in \{+1, -1\}$ and $k \in [K]$ using \eqref{eq:hat_g}
      \FOR {$k \in [K]$ and $i \in \{+1, -1\}$}
        \STATE $W_{i, t + 1}(k) = W_{i, t}(k) e^{\eta \hat{g}_{i, t}(k)}$
      \ENDFOR
    \ENDFOR
  \end{algorithmic}
  \end{algorithm}

\section{\algp{} and Static Regret}
\label{section:3p}

This section describes our base algorithm \algp{} and a high-probability \emph{Static Regret} upper bound on its regret with respect to any fixed benchmark $i^* \in [K]$ (see $\sreg_T$, Sec. \ref{sec:prob}).


\subsection{\algp{} Algorithm}
\label{section:3p_algorithm}

Algorithm \ref{alg:3p} presents the pseudocode for \algp{}. It pits two players against each other, each parallely running its own copy of the standard EXP3 algorithm \citep{auer+02N}. This idea was introduced as the sparring technique by \citet{Ailon+14}, but without any regret guarantees (see Sec. \ref{sec:rel}). We refer to these players as the \textit{row player} and the \textit{column player}, and denote them by $+1$ and $-1$, respectively. 

As in EXP3, these players use weights, $\{W_{+1, t}(k)\}_{k=1}^K$ and $\{W_{-1, t}(k)\}_{k=1}^K$, to compute distributions over $K$ arms, $\bfp_{+1, t} = (p_{+1, t}(1), \dots, p_{+1, t}(K))$ and $\bfp_{-1, t} = (p_{-1, t}(1), \dots, p_{-1, t}(K))$ (Line 5), from which their actions are sampled (Line 7). We use $k_{+1, t}$ and $k_{-1, t}$ to denote the arms chosen by the row and column player at time $t$. The players then update their weights using the observed preference feedback $o_t(k_{+1, t}, k_{-1, t})$ that identifies the winner of a duel between the chosen arms.

The row player considers $r_{+1, t} = o_t(k_{+1, t}, k_{-1, t})$ as its reward at time $t$ whereas the column player uses $r_{-1, t} = 1 - o_t(k_{+1, t}, k_{-1, t})$. The weight update for each player is similar to EXP3 (Line 11) and uses the following expression as an importance-weighted estimate of the rewards. For $i \in \{+1, -1\}$,
\begin{equation}
  \label{eq:hat_g}
  \hat{g}_{i, t}(k) = \begin{cases}
    (r_{i, t} + \beta) / p_{i, t}(k) & \text{if } k = k_{i, t} \\
    \beta / p_{i, t}(k) & \text{otherwise.}
  \end{cases}
\end{equation}
The quantity $\beta \in (0, 1)$ in the expression above prevents $\hat{g}_{i, t}(k)$ from becoming too large when $p_{i, t}(k)$'s are small. This is crucial for obtaining a high-probability regret guarantee (see proof of Lemma \ref{lemma:3p_bound_row_player} in Appendix \ref{appendix:3p_bound_row_player_proof}).

As both players try to maximize their cumulative rewards, they start putting a higher weight on arms that are more likely to win against the choice made by the opponent. An equilibrium is reached when both players choose the optimal arm. The quantity $\gamma \in (0, 1)$ in Line 5 ensures a minimum exploration of all arms and encourages the players to reach the optimal equilibrium strategy. Next, we establish a high-probability static regret guarantee for \algp{}.


\subsection{Analysis: Static Regret Guarantee}
\label{section:3p_analysis}

Given the actions $\{k_{+1, t}, k_{-1, t}\}_{t=1}^T$ taken by the players, the total regret with respect to fixed action $j \in [K]$ is defined as:
\begin{equation*}
    \sreg_T(j) = \sum_{t=1}^T \frac{P_t(j, k_{+1, t}) + P_t(j, k_{-1, t}) - 1}{2}.
\end{equation*}

 The \emph{central idea} in our analysis is to decompose the total regret into the regret incurred by the row ($\sreg_T^{+1}(j)$) and the column ($\sreg_T^{-1}(j)$) player, defined as
\begin{align*}
    \sreg_T^{+1}(j) &= \sum_{t=1}^T P_t(j, k_{-1, t}) - P_t(k_{+1, t}, k_{-1, t}) \\
    \sreg_T^{-1}(j) &= \sum_{t=1}^T P_t(j, k_{+1, t}) - P_t(k_{-1, t}, k_{+1, t}).
\end{align*}
It is easy to see that $\sreg_T(j) = \frac{1}{2} \left( \sreg_T^{+1}(j) + \sreg_T^{-1}(j) \right)$. The next two lemmas establish high-probability upper bounds on $\sreg_T^{+1} \coloneqq \max_{j \in [K]} \sreg_T^{+1}(j)$ and $\sreg_T^{-1} \coloneqq \max_{j \in [K]} \sreg_T^{-1}(j)$.

\begin{restatable}{lemma}{algprowplayer}
  \label{lemma:3p_bound_row_player}
  Let $\eta = \frac{1}{2}\sqrt{\frac{\ln K}{KT}}$, $\beta = \sqrt{\frac{\ln K}{KT}}$, and $\gamma = 2\eta K \geq (1 + \beta) \eta K$. With probability at least $1 - \delta$, 
  $$\sreg_T^{+1} \coloneqq \max_{j\in [K]} \sreg_T^{+1}(j) = O(\sqrt{KT} \ln (K/\delta)).$$
\end{restatable}

\begin{restatable}{lemma}{algpcolumnplayer}
  \label{lemma:3p_bound_column_player}
  Let $\eta$, $\beta$, and $\gamma$ be set as in Lemma \ref{lemma:3p_bound_row_player}. With probability at least $1 - \delta$, 
  $$\sreg_T^{-1} \coloneqq \max_{j\in [K]} \sreg_T^{-1}(j) = O(\sqrt{KT} \ln (K/\delta))$$
\end{restatable}

The following regret bound on $\max_{j \in [K]} \sreg_T(j)$ is now a straightforward consequence of the two lemmas above.

\begin{restatable}[Static-Regret of \algp]{theorem}{algpanalysis}
  \label{theorem:3p_bound}
Set $\eta$, $\beta$, $\gamma$ as in Lem. \ref{lemma:3p_bound_row_player}. Then with probability at least $1 - 2\delta$
  $$\max_{j \in [K]} \sreg_T(j)[\algp{}] = O\left(\sqrt{KT} \ln \left( \frac{K}{\delta} \right) \right)$$
\end{restatable}

Lemmas \ref{lemma:3p_bound_row_player} and \ref{lemma:3p_bound_column_player} have been proved in Appendix \ref{appendix:3p_analysis}. \emph{From the perspective of each agent, the actions of other agents, and hence their own rewards, are adaptively chosen by an adversary.} With this insight, one can borrow techniques from the analysis of the standard EXP3 algorithm to establish the high-probability regret bound in Thm. \ref{theorem:3s_bound}.

\begin{remark}
\label{rem:exp3p}
  The regret guarantee above has been written in terms of any fixed (pure) action $j \in [K]$ in hindsight. However, the adversary can also play any fixed distribution over actions $\bmpi^* \in \Delta_K$, and the same regret guarantee will be valid against $\bmpi^*$. Consequently, this implies that our algorithm achieves an $\tO(\sqrt{KT})$ regret w.r.t. the Nash-Equilibrium (a.k.a the Von-Neumann winner \footnote{Given a preference matrix $\P$, a Von-Neumann winner of $\P \in [0,1]^{K \times K}$ is defined to be a probability distribution $\pi \in \Delta_{K}$, such that $\sum_{i=1}^{n} \pi(i)P(i,b) \ge 1/2, ~~\forall b \in [K]$. (see Eqn. $1$, \citep{dudik+15})}) of the aggregated matrix $\frac{1}{T} \sum_{t =1}^T \P_t$, or even against the best response $\bmpi^*:= \arg\max_{\pi \in \Delta_K}\sreg_T(\pi)[\algp]$.
\end{remark}

With the base algorithm in place, we now modify \algp{} to obtain \algs{} and derive high-probability dynamic regret guarantees for this algorithm.


\begin{algorithm}[t]
  \caption{\algs{}}
  \label{alg:3s}
  \begin{algorithmic}[1]
      \STATE \textbf{Input:} $\eta > 0$, $\alpha, \gamma, \beta \in (0, 1)$
      \STATE Set $W_{+1, 1}(k) = W_{-1, 1}(k) = 1$ for all $k \in [K]$
      \FOR {$t = 1, \dots, T$}
        \FOR {$k \in [K]$ and $i \in \{+1, -1\}$}
          \STATE $p_{i, t}(k) = (1 - \gamma) \frac{W_{i, t}(k)}{\sum_{j=1}^K W_{i, t}(j)} + \frac{\gamma}{K}$
        \ENDFOR
        \STATE For $i \in \{+1, -1\}$, sample $k_{i, t}$ from the distribution $(p_{i, t}(1), \dots, p_{i, t}(K))$
        \STATE Observe preference feedback $o_t(k_{+1, t}, k_{-1, t})$
        \STATE Compute $\hat{g}_{i, t}(k)$ for $i \in \{+1, -1\}$ and $k \in [K]$ using \eqref{eq:hat_g}
        \FOR {$k \in [K]$ and $i \in \{+1, -1\}$}
          \STATE \small{$W_{i, t + 1}(k) = W_{i, t}(k) e^{\eta \hat{g}_{i, t}(k)} + e \alpha \sum_{j=1}^K W_{i, t}(j)$}
        \ENDFOR
      \ENDFOR
\end{algorithmic}
\end{algorithm}

\section{\algs{} and Dynamic Regret}
\label{section:3s}

Competing against a fixed action in hindsight is an interesting but modest goal. In this section, we consider the dynamic regret objective  $\dreg_T(j^T)$ defined with respect to \emph{any sequence of actions} $j^T = (j_1, \dots, j_T)$ (see Section \ref{sec:prob}). It is clear that $\dreg_T(j^T)$ can be $O(T)$ in the worst case unless one imposes further assumptions. Sections \ref{section:switch_nonstationarity} and \ref{section:continuous_nonstationarity} respectively assume the `\sv' and `\cv' model of non-stationarity (see Section \ref{sec:prob}) to derive sub-linear dynamic regret guarantees.

To achieve these bounds, we propose \algs{}, a variant of \algp{}. This algorithm uses the same procedure as \algp{}, except at Line 11, where it instead uses the following weight update rule:
\begin{equation*}
  W_{i, t + 1}(k) = W_{i, t}(k) e^{\eta \hat{g}_{i, t}(k)} + e \alpha W_{i, t}.
\end{equation*}
Here, $\alpha \in (0, 1)$ is a user-specified hyper-parameter and $W_{i, t} \coloneqq \sum_{j=1}^K W_{i, t}(j)$. Algorithm \ref{alg:3s} presents this procedure. We next move to the analysis of \algs{} under the two non-stationarity models.


\subsection{Dynamic Regret Analysis of \algs{} under \sv{} $S$}
\label{section:switch_nonstationarity}

As before, we write $\dreg_T(j^T)$ as $\dreg_T(j^T) = \frac{1}{2}(\dreg_T^{+1}(j^T) + \dreg_T^{-1}(j^T))$, where $\dreg_T^{+1}(j^T)$ and $\dreg_T^{-1}(j^T)$ are the dynamic regret incurred by the row and column player with respect to the sequence $j^T$:
\begin{align*}
  \dreg_T^{+1}(j^T) &= \sum_{t=1}^T P_t(j_t, k_{-1, t}) - P_t(k_{+1, t}, k_{-1, t}) \\
  \dreg_T^{-1}(j^T) &= \sum_{t=1}^T P_t(j_t, k_{+1, t}) - P_t(k_{-1, t}, k_{+1, t}).
\end{align*}
We assume that the sequence $j^T$ has at most $S$ switches, i.e., $1 + \abs{\lbrace 1 \leq \ell < T: j_\ell \neq j_{\ell + 1} \rbrace} \leq S$. The result below bounds $\dreg_T(j^T)$ under two conditions, first where $S$ is known to the algorithm apriori and second where it is unknown.

\begin{restatable}[\algs{} regret under \sv{}]{theorem}{algsanalysis}
  \label{theorem:3s_bound}
  Let $j^T = (j_1, \dots, j_T)$ be any arbitrary sequence of actions with $S$ switches. Then the following bounds hold for \algs{} with probability at least $1 - 2\delta$ when $\alpha = \frac{1}{T}$ and $\gamma = 2 \eta K$:
 
 (1) $\dreg_T(j^T)[\algs] = O \left( \sqrt{SKT} \ln \frac{KT}{\delta} \right)$, and

  (2) $\dreg_T(j^T)[\algs] = O \left( S \sqrt{KT} \ln \frac{KT}{\delta} \right)$, 
  by respectively setting $\beta = \eta = \sqrt{\frac{S}{KT}}$, and $\beta = \eta = \frac{1}{\sqrt{KT}}$.
\end{restatable}

In the first case above, $S$ is known apriori to the algorithm and is used for setting $\beta$ and $\eta$. If $S$ is unknown, the regret grows linearly with $S$.

\begin{remark}
  \label{remark:optimal_jT}
  Theorem \ref{theorem:3s_bound} is valid for any sequence $j^T$ with $S$ switches. In particular, $j^T$ can be constructed as follows. \textbf{(i)} Divide the time horizon $T$ into $S$ sub-intervals $[T_1, \dots, T_2)$, $[T_2, \dots, T_3)$, \dots, $[T_S, \dots, T_{S + 1})$, where $T_1 = 1$ and $T_{S + 1} = T + 1$. \textbf{(ii)} Assume $P_{t}(i, j) = P^s(i, j)$ for all $t \in [T_s, T_{s + 1})$, $s \in [S]$, and $i, j \in [K]$. \textbf{(iii)} Set $j_{T_s} = j_{T_s + 1} = \dots = j_{T_{s + 1} - 1} = j^s$, where
  \begin{equation*}
    j^s = \argmax_{j \in [K]} \sum_{t = T_s}^{T_{s + 1} - 1} \frac{P_s(j, k_{+1, t}) + P_s(j, k_{-1, t}) - 1}{2}.
  \end{equation*}
  In other words, if the winner probabilities remain constant within sub-intervals $[T_s, T_{s + 1})$, then $j_{T_s}, \dots, j_{T_{s + 1} - 1}$ can be set to the best action within this sub-interval in hindsight for each $s \in [S]$. Theorem \ref{theorem:3s_bound} will be valid for such a sequence $j^T$. Therefore, if preferences $\bfP_1, \dots, \bfP_T$ have \sv{} $S$, then, with probability at least $1 - 2\delta$,
  \begin{align*}
    \dreg_T(j^1,\ldots,j^S)[\algs] = \begin{cases}
        O \left(  \sqrt{SKT} \ln \frac{KT}{\delta} \right),\\
        O \left( S\sqrt{KT} \ln \frac{KT}{\delta} \right), 
     \end{cases}
 \end{align*}
 for known and unknown $S$ respectively.
\end{remark}

\begin{remark}
\label{rem:exp3s_pi}
  As in \algp{}, a similar regret guarantee as Theorem \ref{theorem:3s_bound} would still hold for \algs{} even when the adversary plays a fixed distribution over actions in each of the sub-intervals $[T_s, T_{s + 1})$ instead of playing a single fixed action, as explained in Rem. \ref{rem:exp3p}.
\end{remark}

The proof of Theorem \ref{theorem:3s_bound} uses the same regret decomposition strategy as was used in Theorem \ref{theorem:3p_bound}. Lemmas \ref{lemma:3s_row_player_bound} and \ref{lemma:3s_column_player_bound} in Appendix \ref{appendix:3s_analysis_switch} establish high-probability upper bounds on $\dreg_T^{+1}(j^T)$ and $\dreg_T^{-1}(j^T)$. Theorem \ref{theorem:3s_bound} then follows by combining the two lemmas.


\subsection{Dynamic Regret Analysis of \algs{} under \cv{} $V_T$}
\label{section:continuous_nonstationarity}

One can decompose the dynamic regret along the same lines as in Section \ref{section:switch_nonstationarity} to establish a dynamic regret guarantee for \algs{} under the \cv{} model of non-stationarity (Sec. \ref{sec:prob}).

\begin{restatable}[\algs\, regret under \cv]{theorem}{cvalgsanalysis}
\label{theorem:3s_boundcv}
Consider any preference sequence $\bfP_1, \ldots, \bfP_T$ with \cv\, $V_T$. Then, setting $\alpha, \gamma$ same as defined in Thm. \ref{theorem:3s_bound} and $\beta = \eta = \frac{V_T^{1/3}}{4K^{2/3}T^{1/3}}$, for any time-variant benchmark $j^T=(j_1,\ldots,j_T)$, with probability at least $1-\delta$, the dynamic regret bound of \algs{} satisfies $\dreg_T(j^T)[\algs] = O\Big( \big(V_T^{1/3}K^{1/3}T^{2/3} + 4K^{2/3}T^{1/3}V_T^{-1/3} \big)\ln \frac{KT}{\delta}\Big)$.
\end{restatable}

The proof of Theorem \ref{theorem:3s_boundcv} relies on the observation that for any sub-interval $\cI \subseteq [T]$ with limited \cv\, budget, the performance difference between any fixed benchmark $j^* \in [K]$ vs a time-varying benchmark (say a sequence of arms $\{j_\tau^*\}_{\tau \in \cI}$) is bounded. Hence, one can divide the full time horizon $[T]$ into small sub-intervals $\cI_1, \ldots, \cI_S$, each of length say $\Delta$ (a tuning parameter to be chosen appropriately), thus $S = \lceil T/\Delta\rceil$, and apply the results from Thm. \ref{theorem:3s_bound} to bound the learner's regret within each sub-interval against a fixed benchmark. The bound in Thm. \ref{theorem:3s_boundcv} now follows by combining both these factors and by properly adjusting $\Delta$ along with the algorithm parameters $\alpha, \beta, \eta$, and $\gamma$. The complete proof is given in Appendix \ref{app:cont_nst}.

\begin{remark}
Theorems \ref{theorem:3s_bound} and \ref{theorem:3s_boundcv}, together with Theorems \ref{theorem:lower_bound_switch} and \ref{theorem:lower_bound_continuous} respectively, immediately justify the optimality of \algs{} and tightness of our regret analysis in terms of Switching and Continuous Variation. 
\end{remark}

\begin{remark}
\label{rem:high_prob_to_expected_regret}
The regret guarantees in Theorems \ref{theorem:3p_bound}, \ref{theorem:3s_bound}, and \ref{theorem:3s_boundcv} hold with high probability. The corresponding expected regret guarantees of the same order can be derived by setting $\delta = 1/T$. For example, against any sequence $j^T$ with \sv\, $S$, Thm. \ref{theorem:3s_bound} gives $\sreg_T[\algp{}] = O(\sqrt{KT}\ln KT)$ with probability $1 - \delta$, and $\sreg_T[\algp{}] \le T$ with probability $\delta$. Thus choosing $\delta = \frac{1}{T}$ one may conclude that $\rmE{\sreg_T[\algp{}]} = O(\sqrt{KT}\ln KT)$.
\end{remark}

\section{Lower Bound}
\label{section:lower_bound}

In this section, we derive a lower bound on the expected dynamic regret of any algorithm for the \nstdb{} problem. The high-level idea is to construct a class of problem instances $\calV$ and argue that there exists an instance $\nu \in \calV$ and a sequence $j^T = (j_1, \dots, j_T)$ for which $\rmE{\dreg_T(j^T)}$ is sufficiently large (see Theorems \ref{theorem:lower_bound_switch} and \ref{theorem:lower_bound_continuous}) for any algorithm. $\calV$ is defined such that for all instances $\nu \in \calV$:
\begin{enumerate}[leftmargin=*]
  \item The time horizon $[T]$ is divided into $S$ sub-intervals $[T_1, \dots, T_2)$, \dots, $[T_S, \dots, T_{S + 1})$, where $T_1 = 1$ and $T_{S + 1} = T + 1$. For all $s \in [S]$ and $i, j \in [K]$, the winner probabilities satisfy
  \begin{equation*}
    t_1, t_2 \in [T_s, T_{s + 1}) \Rightarrow P_{t_1}(i, j) = P_{t_2}(i, j).
  \end{equation*}
  \item For all $s \in [S]$, an arm $j^s \neq 1$ is chosen uniformly at random as the \textit{condorcet winner} \footnote{Any preference matrix $\P \in [0,1]^{K \times K}$ is said to have a condorcer winner $i^* \in [K]$, if there exists an arm $i^* \in [K]$ s.t. $P(i^*, j) > 0.5$ for all $j \in [K]\sm \{i^*\}$ \citep{Busa14survey}.}, i.e., $P_t(j^s, j) > 0.5$ for all $j \neq j_s$ and for all $t \in [T_s, T_{s + 1})$.
  \item With $\epsilon > 0$ as a small constant, for all $s \in [S]$, $t \in [T_s, T_{s + 1})$, and $i < j$, the winner probabilities are:
\end{enumerate}
\begin{equation}
    \label{eq:env_specification}
    P_t(i, j) = \begin{cases}
        \frac{1}{2} + \epsilon & \text{ if } i = j^s \lor \left( i = 1 \land j \neq j^s \right), \\
        \frac{1}{2} - \epsilon & \text{ if } j \in \lbrace 1, j^s \rbrace, \\
        \frac{1}{2} & \text{ otherwise.}
    \end{cases}
\end{equation}
In other words, an instance $\nu \in \calV$ is created by choosing a condorcet winner $j^s \in [K] \backslash \lbrace 1 \rbrace$ for each sub-interval $[T_s, T_{s + 1})$. In sub-interval $[T_s, T_{s + 1})$, arm $j^s$ beats every arm with probability $0.5 + \epsilon$, arm $1$ beats every arm except $j^s$ with probability $0.5 + \epsilon$, and all other arms beat each other with equal probability.

For an instance $\nu \in \calV$, define the sequence of actions $j^T(\nu)$ such that $j_t(\nu) = j^s$ for all $t \in [T_s, T_{s + 1})$. Here, $j_t(\nu)$ is the $t^{th}$ element of $j^T(\nu)$ and $j^s$ is the condorcet winner in $\nu$ during the sub-interval $[T_s, T_{s + 1})$. The results below establish lower bounds on $\rmE{\dreg_T(j^T(\nu))}$ under \sv{} and \cv{} model of non-stationarity.

\begin{restatable}[Dynamic Regret Lower Bound for \sv]{theorem}{lowerboundswitch}
  \label{theorem:lower_bound_switch}
  Assume that $K \geq 3$ and $\epsilon \in (0, 1/4)$ in \eqref{eq:env_specification}. With $\calV$ and $j^T(\nu)$ defined above, there exists an instance $\nu \in \calV$ such that: 
  $$\rmE{\dreg_T(j^T(\nu))} = \Omega(\sqrt{SKT}),$$ for any non-stationary dueling bandits algorithm.
\end{restatable}

Given a change budget $V_T$ and the number of switches $S$, $\epsilon$ can be set to ensure that the total change in $\bfP_t$'s across $T$ steps is at most $V_T$ (see Appendix \ref{appendix:lower_bound_continuous}). Using this value of $\epsilon$ yields the following result:

\begin{restatable}[Dynamic Regret Lower Bound for \cv]{theorem}{lowerboundcontinuous}
  \label{theorem:lower_bound_continuous}
  Assume that $K \geq 3$. With $\calV$ and $j^T(\nu)$ defined above, there exists an instance $\nu \in \calV$ such that: 
  $$\rmE{\dreg_T(j^T(\nu))} = \Omega((KV_T)^{1/3} T^{2/3}),$$ for any non-stationary dueling bandits algorithm.
\end{restatable}

The proofs of Theorems \ref{theorem:lower_bound_switch} and \ref{theorem:lower_bound_continuous} are given in Appendix \ref{appendix:lower_bound}. The key idea is to establish a lower bound on the cumulative regret within each sub-interval $[T_s, T_{s + 1})$. This can be done by upper-bounding the number of times arms $1$ and $j^s$ (these arms are more likely to win against others) are pulled in expectation within each sub-interval $[T_s, T_{s + 1})$.

\begin{algorithm}[t]
  \caption{\algdexp{}}
  \label{alg:dexp3_whp}
  \begin{algorithmic}[1]
      \STATE \textbf{Input:} $\eta > 0$, $\alpha, \gamma, \beta \in (0, 1)$
      \STATE Set $W_{+1, 1}(k) = W_{-1, 1}(k) = 1$ for all $k \in [K]$
      \FOR {$t = 1, \dots, T$}
        \FOR {$k \in [K]$ and $i \in \{+1, -1\}$}
          \STATE $p_{i, t}(k) = (1 - \gamma) \frac{W_{i, t}(k)}{\sum_{j=1}^K W_{i, t}(j)} + \frac{\gamma}{K}$
        \ENDFOR
        \STATE For $i \in \{+1, -1\}$, sample $k_{i, t}$ from the distribution $(p_{i, t}(1), \dots, p_{i, t}(K))$
        \STATE Observe preference feedback $o_t(k_{+1, t}, k_{-1, t})$
        \STATE Compute $s'_{i, t}(k)$ for $i \in \{+1, -1\}$ and $k \in [K]$ as: \tiny{$$s'_{i, t}(k) = \frac{\1(k_{i,t} = k)}{K p_{i, t}(k)}\sum_{j \in [K]}\frac{\1(k_{-i,t} = j) o_t(k, j)}{p_{-i, t}(j)} + \frac{\beta}{p_{i, t}(k)}$$}\normalsize
        \FOR {$k \in [K]$ and $i \in \{+1, -1\}$}
          \STATE \small{$W_{i, t + 1}(k) = W_{i, t}(k) e^{\eta s' _{i, t}(k)} + e \alpha \sum_{j=1}^K W_{i, t}(j)$}
        \ENDFOR
      \ENDFOR
\end{algorithmic}
\end{algorithm}

\section{Dynamic Regret with Borda Scores}
\label{sec:ub_borda}

In this section, we briefly switch to a different regret measure that uses the so-called Borda scores. The \emph{Borda score} of an arm $i \in [K]$ with respect to the preference matrix $\P_t$ at time $t$ measures the average probability of arm $i$ winning against a randomly chosen opponent $j \in [K]\backslash\{i\}$ at time $t$:
\begin{align*}
    b_t(i) \coloneqq \frac{1}{K-1}\sum_{j \neq i} P_t(i,j).  
\end{align*}
While the ideas presented here apply equally well to the static Borda regret objective (regret with respect to the arm with highest cumulative Borda score) that was already addressed in \cite{ADB}, we only focus on the more challenging dynamic Borda regret objective:
\small
\begin{align}
  \label{eq:dbreg}
	\dbreg_T(j^T) \coloneqq \sum_{t=1}^T \frac{2 b_t(j_t) -  b_t(k_{+1, t}) - b_t(k_{-1, t})}{2},
\end{align}
\normalsize
where, as before, $j^T = (j_1, \dots, j_T)$ is an arbitrary sequence of arms, and $k_{+1, t}$ and $k_{-1, t}$, respectively, are the arms chosen by the row and column player at time $t$.

Algorithm \ref{alg:dexp3_whp} describes a variant of \algs{} that uses a different arm rewards estimate instead of $\hat{g}_{i, t}(k)$, but is otherwise same as \algs{}. We call this algorithm \algdexp{}. The reward estimate in \algdexp{} was also used in \citet{ADB} where the authors studied only the static Borda regret objective. The result below establishes an upper bound on the more challenging dynamic regret $\dbreg_T(j^T)$ of Alg.~\ref{alg:dexp3_whp} under the switching variation model of non-stationarity. The proof uses the same regret-decomposition idea as Theorem \ref{theorem:3s_bound}, but differs in technical details to accommodate the change in the reward estimate. See Appendix \ref{appendix:analysis_with_borda_scores} for details.

\begin{restatable}[\algdexp{} regret under \sv{}]{theorem}{algdexpanalysis}
  \label{theorem:3s_borda_bound}
  Let $j^T = (j_1, \dots, j_T)$ be any sequence of actions with $S$ switches. Then, with probability at least $1 - \delta$, $\dbreg_T(j^T)[\algdexp] = \tO(S^{1/6} K^{-1/3} T^{5/6} + S^{1/2} K^{1/3} T^{2/3})$, when $\eta = \left( \frac{S\ln K}{T \sqrt{2K}} \right)^{2/3}$, $\beta = \frac{S^{1/3} \sqrt{\ln (2K/\delta)}}{(2\eta)^{1/4} K^{3/4} \sqrt{T}}$, $\gamma = \sqrt{2 \eta K}$, and $\alpha = \frac{1}{KT}$.
\end{restatable}

It is worth noting that, as in Sec. \ref{section:continuous_nonstationarity}, one can also extend Thm. \ref{theorem:3s_borda_bound} to analyze the continuous variation dynamic Borda regret of Alg. \ref{alg:dexp3_whp} using similar arguments as in the proof of Thm. \ref{theorem:3s_boundcv}. We skip this analysis to avoid redundancy.

\section{Experiments}
\label{section:experiments}

We now present numerical results to complement our theoretical findings and compare with baselines. Section \ref{section:experiments_static_regret} studies static regret for \algp{} against the best fixed action in hindsight. Section \ref{section:experiments_dynamic_regret} studies dynamic regret for \algs{} under \sv{} and \cv{} models of non-stationarity.

\paragraph{Baseline algorithms:} We compare our algorithms against two baselines: \rand{} and \rex{}. \rand{} independently samples arms $k_{+1, t}$ and $k_{-1, t}$ from a uniform distribution over $K$ arms at each time $t$. \rex{} \citep{Adv_DB} uses a time-varying distribution $\bfp_t \in \Delta_K$ and independently samples $k_{+1, t}, k_{-1, t} \sim \bfp_t$.

The values of parameters $\alpha$, $\beta$, $\eta$, and $\gamma$ for \algp{} and \algs{} were set in accordance with Theorems \ref{theorem:3p_bound} and \ref{theorem:3s_bound} (or \ref{theorem:3s_boundcv} as appropriate from the context), respectively. \rex{} uses a single parameter $\gamma$ and we set its value to $\gamma = \sqrt{\frac{2K \ln K}{e T}}$, as recommended in \citet{Adv_DB}.


\subsection{\algp{}: Static Regret}
\label{section:experiments_static_regret}

Recall that $P_t(i, j)$ is the probability with which arm $i$ beats arm $j$ at time $t$. We simulate an environment where these values follow a Gaussian random walk. That is, for every $t \in [T]$ and $i < j$,
\begin{align*}
  P_{t + 1}(i, j) = P_t(i, j) + \epsilon_t(i, j),
\end{align*}
where $\epsilon_t(i, j) \sim \calN(0, 0.002)$\footnote{$\calN(\mu, \sigma)$ denotes a Gaussian distribution with mean $\mu \in \bbR$ and standard deviation $\sigma > 0$.}. We ensure that $P_t(i, i) = 0.5$, $P_t(j, i) = 1 - P_t(i, j)$, and clip the entries of $\bfP_t$ in the range $[0, 1]$ so that they represent valid probabilities. The initial values $P_1(i, j) \sim \mathrm{Uniform}(0, 1)$.

Let $j^* = \argmax_{j \in [K]} \sum_{t = 1}^T (P_t(j, k_{+1, t}) + P_t(j, k_{-1, t}) - 1)/2$ be the best fixed action in hindsight. Figure \ref{fig:static_vs_T} sets $K = 10$ and plots the static regret $\sreg_T$ (with respect to fixed arm $j^*$) incurred by \algp{}, \rand{}, and \rex{} for different values of $T$ in the range $10^4$ and $10^6$. The dotted line in the plot represents $O(\sqrt{T})$ growth. It can be seen that $\sreg_T$ grows as $O(\sqrt{T})$ for \algp{} and \rex{}. Moreover, \algp{} outperforms \rex{} when $T$ is large. Note that $\sreg_T$ can indeed be negative as $\P_t$ changes over time and no fixed benchmark (arm) could be competitive against the learner's dueling sequence $\{(k_{+1,t},k_{-1,t})\}_{t \in [T]}$.

Figure \ref{fig:static_vs_K} similarly fixes $T = 10^6$ and plots $\sreg_T$ for different values of $K$. The dotted line in this plot represents $O(\sqrt{K})$ growth and shows that $\sreg_T$ grows as $O(\sqrt{K})$ for \algp{} and \rex{}. As $K$ increases, \rand{}'s regret slightly decreases as it becomes harder to have a fixed action that performs well across time. Moreover, as expected, the gap between the performance of \algp{} (and \rex{}) and \rand{} decreases as $K$ becomes large (recall that $T$ is fixed).


\subsection{\algs{}: Dynamic Regret}
\label{section:experiments_dynamic_regret}

Next, we turn to \algs{} and focus on dynamic regret under the two notions of non-stationarity (see Sec. \ref{sec:prob}).

\paragraph{\sv{} non-stationarity:} For these experiments, we use a similar Gaussian random walk model for the probabilities $P_t(i, j)$ as in Section \ref{section:experiments_static_regret}, but change the values only after every $\Delta = T / (S - 1)$ steps. That is,
\begin{align*}
  P_{t + 1}(i, j) = \begin{cases}
    P_t(i, j) + \epsilon_t(i, j) & \text{if } \exists c \in \bbN : t = \Delta c \\
    P_t(i, j) & \text{otherwise.}
  \end{cases}
\end{align*}
Here, $\bbN$ is the set of all natural numbers. We sample $\epsilon_t(i, j) \sim \calN(0, 0.05)$ where we use a higher standard deviation as compared to Section \ref{section:experiments_static_regret} as the changes in $P_t(i, j)$ happen less frequently. As before, we set $P_t(i, i) = 0.5$, $P_t(j, i) = 1 - P_t(i, j)$, and clip all values in the range $[0, 1]$. 

The dynamic regret is measured with respect to the sequence $j^T$ defined in Remark \ref{remark:optimal_jT}. This sequence considers the best fixed action in hindsight for each of the sub-intervals of length $\Delta$ in which the probability values remain unchanged. See Remark \ref{remark:optimal_jT} for details.

Figure \ref{fig:switch_vs_T} fixes $K = S = 10$ and plots $\dreg_T(j^T)$ as a function of $T$. It can be verified that the growth of $\dreg_T(j^T)$ is $O(\sqrt{T})$ for \algs{}, as established in Theorem \ref{theorem:3s_bound}. Moreover, \algs{} slightly outperforms both \algp{} and \rex{} for large $T$. Figure \ref{fig:switch_vs_K} fixes $S = 10$, $T = 10^6$, and plots $\dreg_T(j^T)$ as a function of $K$, confirming that $\dreg_T(j^T)$ indeed grows as $O(\sqrt{K})$. 

\begin{figure}[h]
  \subfloat[Regret vs $T$]{\includegraphics[width=0.5\linewidth]{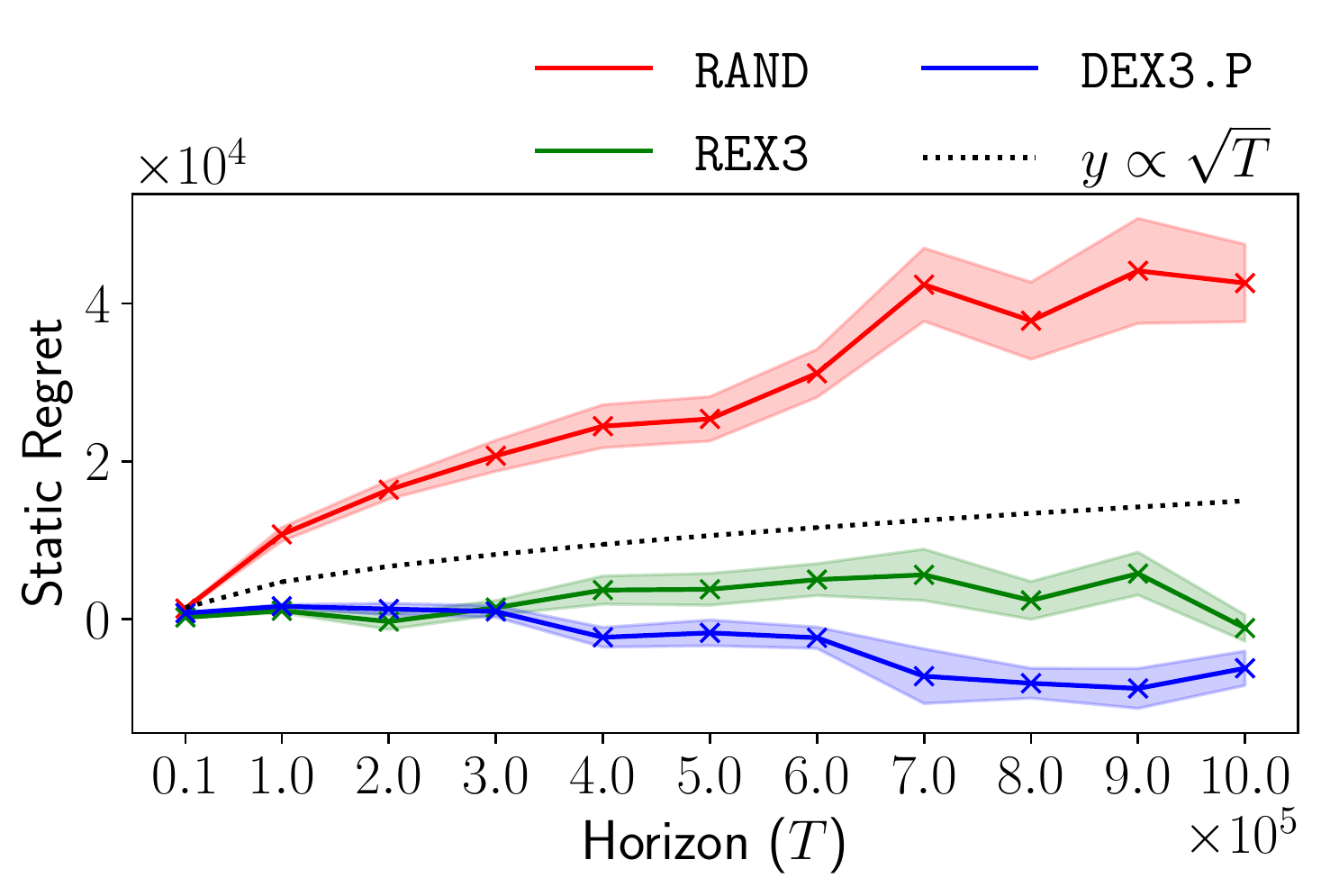}\label{fig:static_vs_T}}%
  \subfloat[Regret vs $K$]{\includegraphics[width=0.5\linewidth]{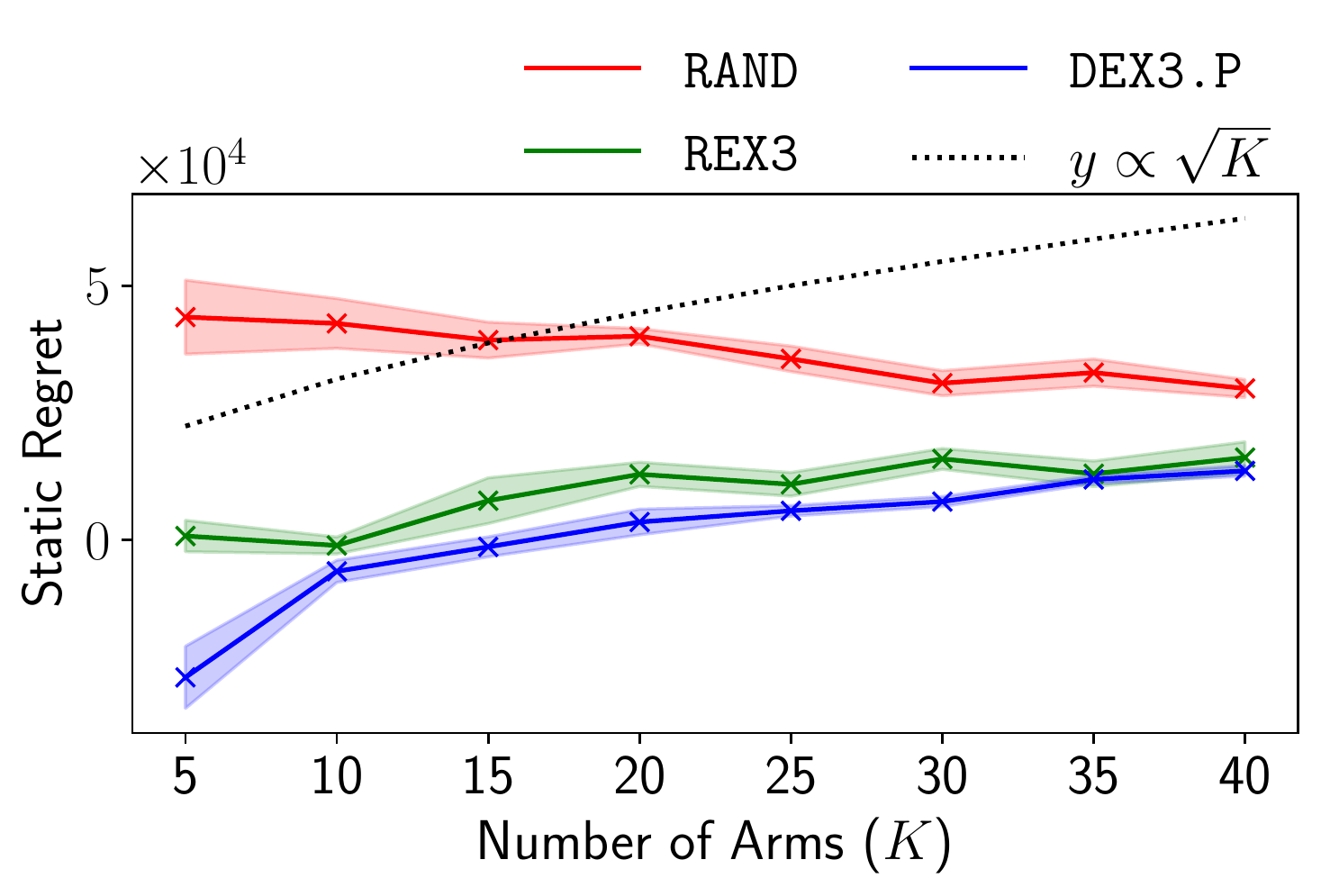}\label{fig:static_vs_K}}
  \caption{Comparison between \algp{}, \rand{}, and \rex{} in terms of static regret with respect to the best fixed action in hindsight. The shaded region represents standard deviation across $10$ independent runs.}
  \label{fig:static_regret}
\end{figure}
\paragraph{\cv{} non-stationarity:} Recall that $\Delta_T$ denotes a $T$-simplex and let $\bmmu \in \Delta_T$. Given \cv{} $V_T$, we construct a sequence of probability values $P_t(i, j)$ such that for all $t \in [T]$
\begin{equation*}
  \max_{i, j \in [K]} \abs{P_{t + 1}(i, j) - P_t(i, j)} = \mu_t V_T.
\end{equation*}
Here $\mu_t$ denotes the $t^{th}$ element of $\bmmu$. As before, we sample $P_1(i, j) \sim \mathrm{Uniform}(0, 1)$ and ensure that $P_t(i, i) = 0.5$, $P_t(j, i) = 1 - P_t(i, j)$, and $P_t(i, j) \in [0, 1]$ for all $t$. For the plots in Figure \ref{fig:continuous_dynamic_regret}, we sample $\bmmu \sim \mathrm{Dirichlet}(1, \dots, 1)$. The dynamic regret is measured with respect to the best action at each time step, i.e., $j^T = (j_1, \dots, j_T)$ is such that 
\begin{equation}
  \label{eq:optimal_action_sequence}
  j_t = \argmax_{j \in [K]} \frac{1}{2} (P_t(j, k_{+1, t}) + P_t(j, k_{-1, t}) - 1).
\end{equation}

\begin{figure}[H]
  \subfloat[Regret vs $T$]{\includegraphics[width=0.5\linewidth]{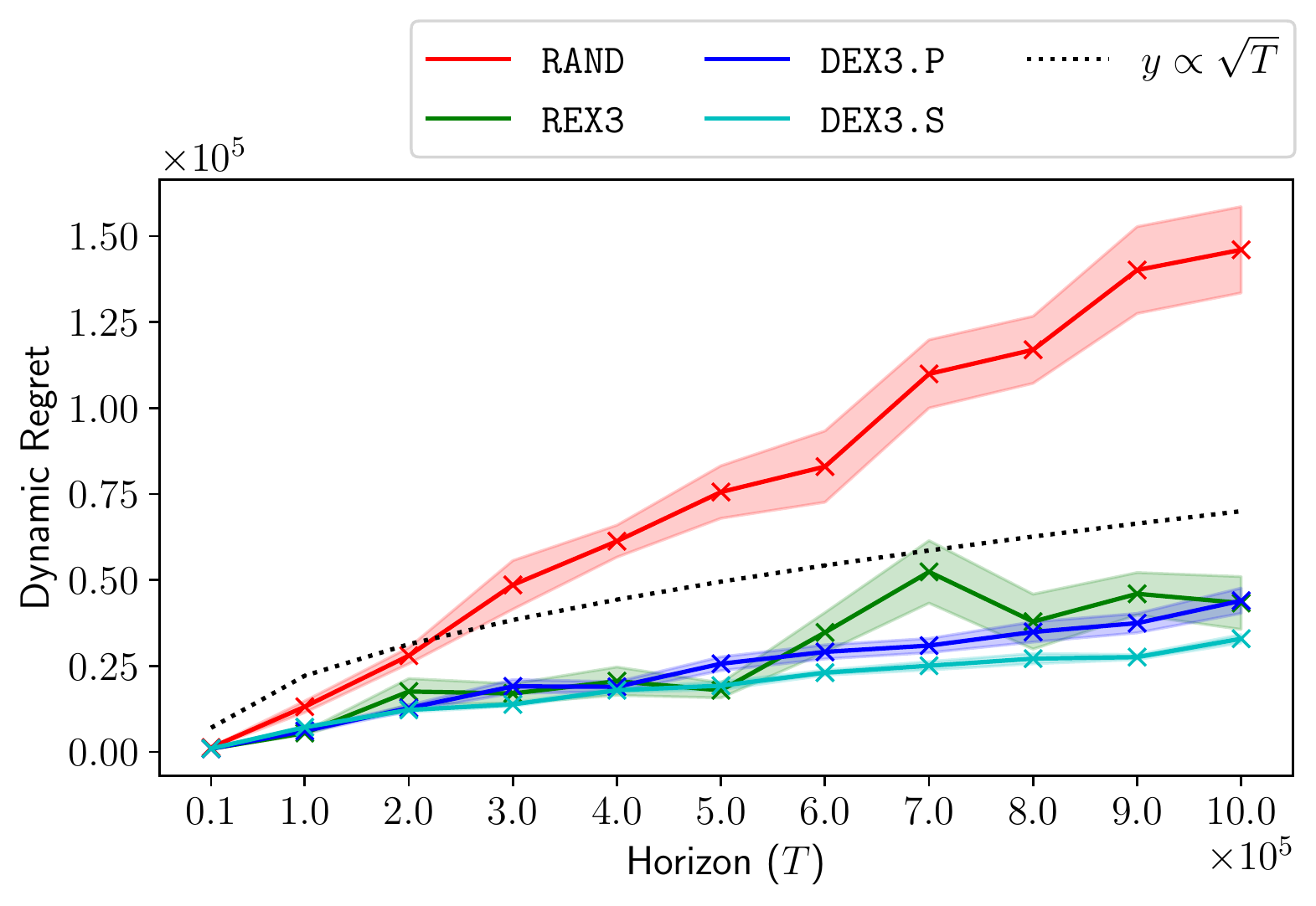}\label{fig:switch_vs_T}}%
  \subfloat[Regret vs $K$]{\includegraphics[width=0.5\linewidth]{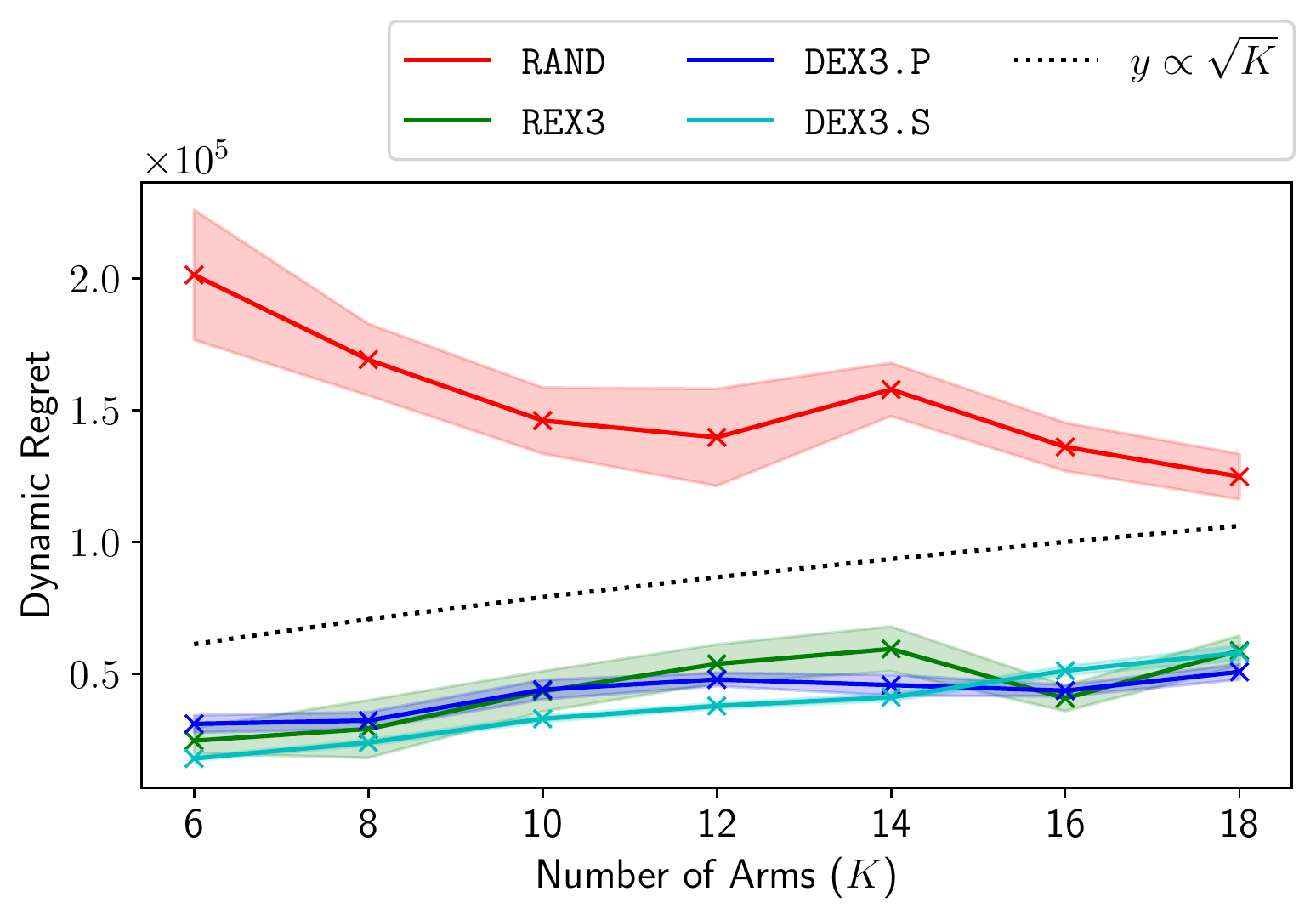}\label{fig:switch_vs_K}}
  \caption{Dynamic regret under \sv{} non-stationarity with respect to sequence $j^T$ defined in Remark \ref{remark:optimal_jT}. The shared region represents standard deviation across $10$ independent runs.}
  \label{fig:switch_dynamic_regret}
\end{figure}
\begin{figure}[H]
  \subfloat[Regret vs $T$]{\includegraphics[width=0.5\linewidth]{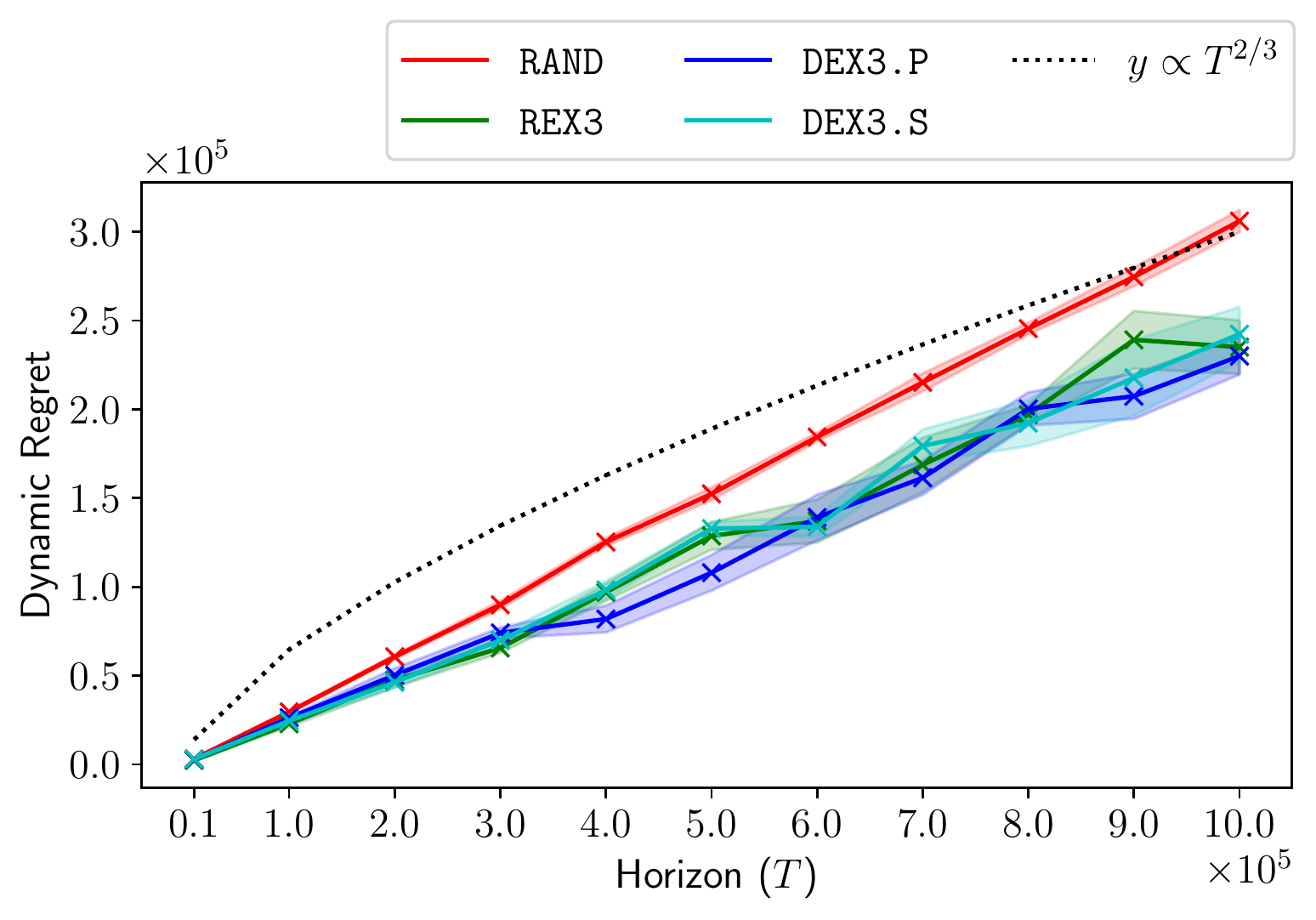}\label{fig:continuous_vs_T}}%
  \subfloat[Regret vs $K$]{\includegraphics[width=0.5\linewidth]{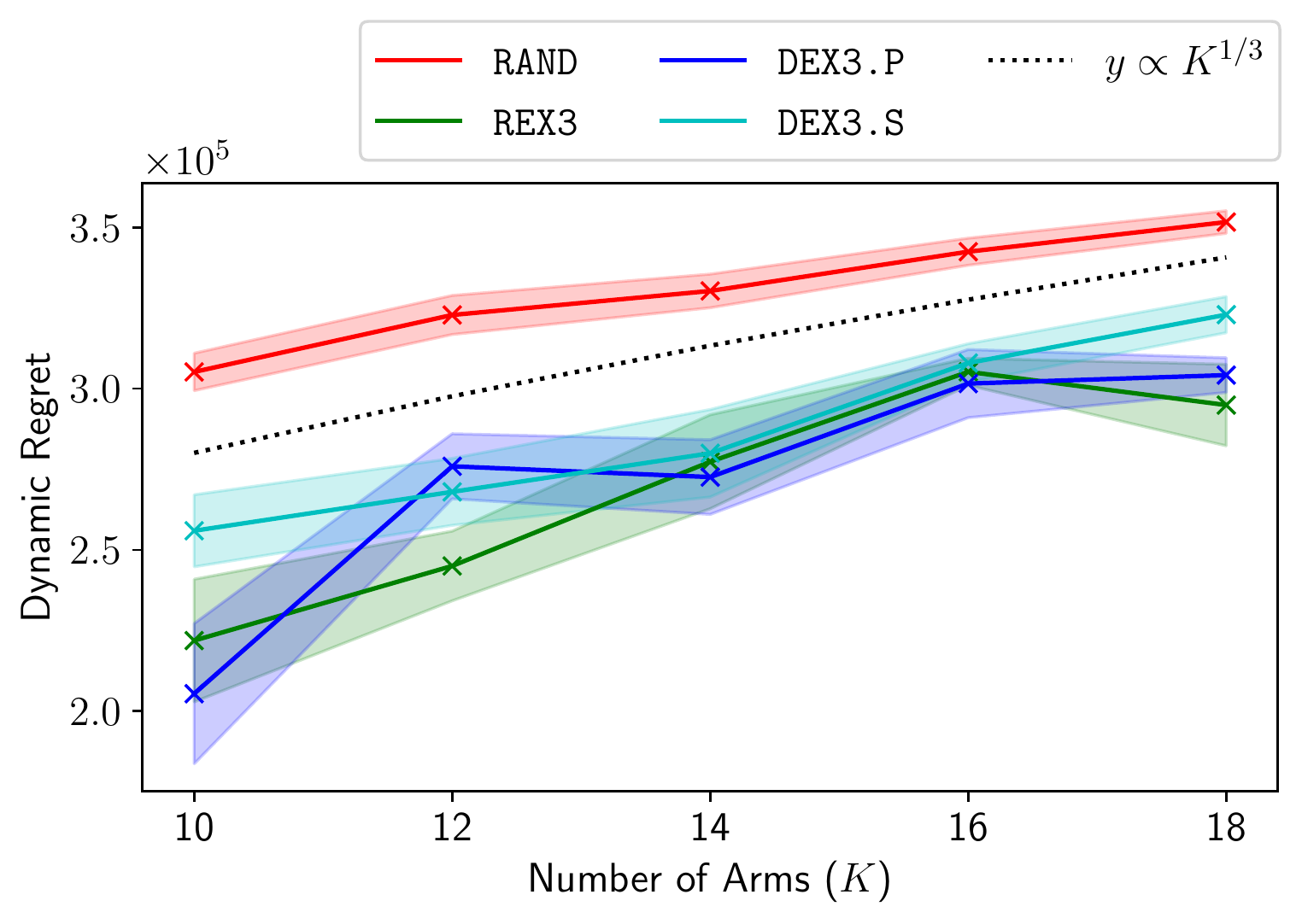}\label{fig:continuous_vs_K}}
  \caption{Dynamic regret under \cv{} non-stationarity w.r.t. sequence $j^T$ (Eqn. \eqref{eq:optimal_action_sequence}). The shared region represents standard deviation across $10$ independent runs.}
  \label{fig:continuous_dynamic_regret}
\end{figure}

Figures \ref{fig:continuous_vs_T} and \ref{fig:continuous_vs_K} show that $\dreg_T(j^T)$ grows as $O(T^{2/3})$ and $O(K^{1/3})$, respectively, as established in Theorem \ref{theorem:3s_boundcv}. Figure \ref{fig:continuous_vs_T} uses $K = V_T = 10$ and Figure \ref{fig:continuous_vs_K} uses $V_T = 10$ and $T = 10^6$. Finally, we would like to emphasize that while the dynamic regret for \rex{} is close to \algs{} in our experiments, theoretically \rex{} only comes with a static regret guarantee, whereas we have worst-case dynamic regret guarantees for \algs{}.


\section{Discussions}
\label{sec:concl}

We analyzed the complexity of {regret minimization} in non-stationary dueling bandits. We first proposed an $\tO(\sqrt{KT})$ adversarial dueling bandit algorithm which is shown to be provably competitive against any fixed benchmark (arm-distribution). This algorithm was further extended to yield $\tO(\sqrt{SKT})$ and $\tO({V_T^{1/3}K^{1/3}T^{2/3}})$ dynamic-regret algorithms for \sv\, $(S)$ and \cv\, $(V_T)$ non-stationarities. The optimality and effectiveness of our proposed techniques were justified through matching lower bound analysis and extensive empirical evaluations. It was also demonstrated that the ideas in the proposed algorithm apply to a different, Borda score based, regret objective. To the best of our knowledge, we are the first to analyze dynamic regret guarantees for dueling bandits in non-stationary environments.

\textbf{Future Works.} 
The potential future scope of this work is enormous, this being a nearly unexplored direction in preference bandits literature. An obvious next step is to formulate and study the \nstdb\, problem complexity for other meaningful definitions of non-stationarities.  Extending the results to contextual scenarios and subsetwise preference feedback setup would also be of much practical importance. Another useful generalization could be to make the algorithms delay tolerant where the  preference feedback are revealed only after a certain unknown delay \citep{delay1, delay2} or analyzing the dynamic regret with paid observations \citep{paid1}.

\newpage

\bibliographystyle{plainnat}
\bibliography{db_refs,nstDB}


\newpage
\onecolumn
\allowdisplaybreaks

\par\noindent\rule{\textwidth}{1pt}
\begin{center}
\Large\textbf{Supplementary Material: \\ \papertitle}
\end{center}
\par\noindent\rule{\textwidth}{0.4pt}

\appendix


\section{Missing details from the Static Regret Analysis of \algp{} (\Cref{section:3s})}
\label{appendix:3p_analysis}

The proof of Lemmas \ref{lemma:3p_bound_row_player} and \ref{lemma:3p_bound_column_player} use similar arguments. Therefore, we only present the proof of Lemma \ref{lemma:3p_bound_row_player} in Section \ref{appendix:3p_bound_row_player_proof}. Let us begin with a technical result.

\begin{lemma}
  \label{lemma:g_hat_upper_bounds_p_row_player}
  For any $\delta > 0$, $\beta \in (0, 1)$, with probability at least $1 - \delta/2$, $$\forall k \in [K], \;\;\;\; \sum_{t=1}^T \hat{g}_{+1, t}(k) \geq \sum_{t = 1}^T P_t(k, k_{-1, t}) - \frac{\ln(2K/\delta)}{\beta}.$$
\end{lemma}

\begin{proof}

Let $\calF^{+1}_{t - 1} = \sigma\left( \lbrace \bfp_{i, s}, k_{i, s}, o_s \rbrace_{s < t, i \in \lbrace +1, -1 \rbrace} \; \cup \; \lbrace \bfp_{+1, t}, \bfp_{-1, t}, k_{-1, t} \rbrace \right)$ be the $\sigma$-algebra generated by using all information till time $t$, except the choice of the row player's action $k_{+1, t}$ and the feedback from the environment $o_t(k_{+1, t}, k_{-1, t})$ at time $t$. Note that
\begin{equation}
    \label{eq:g_hat_expectation}
    \rmE{ \hat{g}_{+1, t}(k) \vert \calF^{+1}_{t - 1}} = P_t(k, k_{-1, t}) + \frac{\beta}{p_{+1, t}(k)}.
\end{equation}

By Markov's inequality,
\begin{align*}
    \rmP{e^{-\beta \sum_{t=1}^T \left(\hat{g}_{+1, t}(k) -  P_t(k, k_{-1, t})\right)} \leq \frac{2}{\delta}} &\geq 1 - \rmP{e^{-\beta \sum_{t=1}^T \left(\hat{g}_{+1, t}(k) -  P_t(k, k_{-1, t})\right)} \geq \frac{2}{\delta}} \\
    &\geq 1 - \frac{\delta}{2} \rmE{\exp\left( \beta \sum_{t=1}^T \left(P_t(k, k_{-1, t}) - \hat{g}_{+1, t}(k) \right) \right)}.
\end{align*}
We only need to show that $\rmE{\exp\left( \beta \sum_{t=1}^T \left(P_t(k, k_{-1, t}) - \hat{g}_{+1, t}(k) \right) \right)} \leq 1$ to finish the proof. Consider the quantity $\rmE{ \exp\left(\beta(P_t(k, k_{-1, t}) - \hat{g}_{+1, t}(k)) \right) \vert \calF^{+1}_{t-1}}$.
\begin{align*}
    \rmE{ e^{\beta(P_t(k, k_{-1, t}) - \hat{g}_{+1, t}(k))}\vert \calF^{+1}_{t-1}} &= \rmE{ e^{\beta(P_t(k, k_{-1, t}) - \hat{g}_{+1, t}(k) + \beta/p_{+1, t}(k))} \Bigg\vert \calF^{+1}_{t-1}} e^{-\frac{\beta^2}{p_{+1, t}(k)}} \\
    &\overset{(a)}{\leq} \mathrm{E}\Bigg[1 + \beta\left(P_t(k, k_{-1, t}) - \hat{g}_{+1, t}(k) + \frac{\beta}{p_{+1, t}(k)}\right) + \\
    &\hspace{1.1cm} \beta^2\left(P_t(k, k_{-1, t}) - \hat{g}_{+1, t}(k) + \frac{\beta}{p_{+1, t}(k)}\right)^2 \Bigg\vert \calF^{+1}_{t-1} \Bigg] e^{-\frac{\beta^2}{p_{+1, t}(k)}} \\
    &\overset{(b)}{=} \left[1 + \beta^2\rmE{\left(P_t(k, k_{-1, t}) - \hat{g}_{+1, t}(k) + \frac{\beta}{p_{+1, t}(k)}\right)^2 \Bigg\vert \calF^{+1}_{t-1}} \right] e^{-\frac{\beta^2}{p_{+1, t}(k)}} \\
    &\overset{(c)}{=} \left[1 + \beta^2 \mathrm{Var}\left( \hat{g}_{+1, t}(k) \vert \calF^{+1}_{t-1} \right) \right] e^{-\frac{\beta^2}{p_{+1, t}(k)}} \\
    &= \left[1 + \beta^2 \mathrm{Var}\left( \frac{o_t(k, k_{-1, t})\bfone{k_{+1, t} = k}}{p_{+1, t}(k)} \Bigg\vert \calF^{+1}_{t-1} \right) \right] e^{-\frac{\beta^2}{p_{+1, t}(k)}} \\
    &\leq \left[1 + \beta^2 \rmE{\left(\frac{o_t(k, k_{-1, t})\bfone{k_{+1, t} = k}}{p_{+1, t}(k)} \right)^2 \Bigg\vert \calF^{+1}_{t-1}} \right] e^{-\frac{\beta^2}{p_{+1, t}(k)}} \\
    &= \left[1 + \frac{\beta^2}{p_{+1, t}^2(k)} \rmE{o_t(k, k_{-1, t})\bfone{k_{+1, t} = k} \vert \calF^{+1}_{t-1}} \right] e^{-\frac{\beta^2}{p_{+1, t}(k)}} \\
    &\overset{(d)}{\leq} \left[1 + \frac{\beta^2}{p_{+1, t}(k)}\right] e^{-\frac{\beta^2}{p_{+1, t}(k)}} \\
    &\overset{(e)}{\leq} 1.
\end{align*}
Here, $(a)$ follows as $\beta(P_t(k, k_{-1, t}) - \hat{g}_{+1, t}(k) + \beta/p_{+1, t}(k)) \leq 1$ because $\hat{g}_{+1, t}(k) \geq \beta/p_{+1, t}(k)$ by \eqref{eq:hat_g}. Thus, we can use the identity $e^x \leq 1 + x + x^2$ for $x \leq 1$. $(b)$ and $(c)$ follow from \eqref{eq:g_hat_expectation}, $(d)$ uses the fact that $o_t(i, j) \leq 1$, and $(e)$ follows as $1 + x \leq e^x$. Now, by induction,
\begin{align*}
    \rmE{e^{\beta \sum_{t=1}^T \left(P_t(k, k_{-1, t}) - \hat{g}_{+1, t}(k) \right)}} &= \rmE{\rmE{e^{\beta \left(P_T(k, k_{-1, T}) - \hat{g}_{+1, T}(k) \right)} \Big\vert \calF^{+1}_{T-1}} e^{\beta \sum_{t=1}^{T-1} \left(P_t(k, k_{-1, t}) - \hat{g}_{+1, t}(k) \right)}} \\
    &\leq \rmE{e^{\beta \sum_{t=1}^{T-1} \left(P_t(k, k_{-1, t}) - \hat{g}_{+1, t}(k) \right)}} \leq \hdots \leq 1.
\end{align*}
Setting $\delta = \delta/K$ and taking a union bound over all $k \in [K]$ finishes the proof.

\end{proof}


\subsection{Proof of Lemma \ref{lemma:3p_bound_row_player}}
\label{appendix:3p_bound_row_player_proof}

\algprowplayer*

\begin{proof}

Recall from Algorithm \ref{alg:3p} that $W_{+1, t}(k) = \exp(\eta\sum_{s=1}^{t - 1} \hat{g}_{+1, s}(k))$ and define $W_{+1, t} = \sum_{k = 1}^K W_{+1, t}(k)$. Then,
\begin{align*}
    W_{+1, t + 1} &= \sum_{k = 1}^K W_{+1, t + 1}(k) \\
    &= \sum_{k = 1}^K W_{+1, t}(k) \exp(\eta \hat{g}_{+1, t}(k)) \\
    &= W_{+1, t} \sum_{k=1}^K \frac{W_{+1, t}(k)}{W_{+1, t}} \exp(\eta \hat{g}_{+1, t}(k)) \\
    &\overset{(a)}{=} W_{+1, t} \sum_{k=1}^K \frac{p_{+1, t}(k) - \gamma/K}{1 - \gamma} \exp(\eta \hat{g}_{+1, t}(k)) \\
    &\overset{(b)}{\leq} W_{+1, t} \sum_{k=1}^K \frac{p_{+1, t}(k) - \gamma/K}{1 - \gamma} (1 + \eta \hat{g}_{+1, t}(k) + \eta^2 \hat{g}^2_{+1, t}(k)) \\
    &\leq W_{+1, t} \left(1 + \frac{\eta}{1 - \gamma} \sum_{k=1}^K p_{+1, t}(k) \hat{g}_{+1, t}(k) + \frac{\eta^2}{1 - \gamma} \sum_{k=1}^K p_{+1, t}(k) \hat{g}^2_{+1, t}(k) \right) \\
    &\overset{(c)}{\leq} W_{+1, t} \exp\left(\frac{\eta}{1 - \gamma} \sum_{k=1}^K p_{+1, t}(k) \hat{g}_{+1, t}(k) + \frac{\eta^2}{1 - \gamma} \sum_{k=1}^K p_{+1, t}(k) \hat{g}^2_{+1, t}(k) \right)
\end{align*}
Here, $(a)$ follows from the definition of $p_{+1, t}(k)$ in Algorithm \ref{alg:3p}. Note that, by assumption on $\eta$, $\beta$, $\gamma$, and $K$
\begin{align*}
    \eta \hat{g}_{+1, t}(k) \leq \eta \frac{1 + \beta}{p_{+1, t}(k)} \leq \eta K\frac{1 + \beta}{\gamma} \leq 1.
\end{align*}
Thus, $(b)$ follows from the inequality $e^x \leq 1 + x + x^2$ if $x \leq 1$. $(c)$ follows from the inequality $1 + x \leq e^x$. We now have,
\begin{align*}
    \ln \frac{W_{+1, t + 1}}{W_{+1, t}} \leq \frac{\eta}{1 - \gamma} \sum_{k=1}^K p_{+1, t}(k) \hat{g}_{+1, t}(k) + \frac{\eta^2}{1 - \gamma} \sum_{k=1}^K p_{+1, t}(k) \hat{g}^2_{+1, t}(k).
\end{align*}
Summing over $t = 1, 2, \dots, T$, we get
\begin{equation}
    \label{eq:w_ration_upper_bound_row_player}
    \ln \frac{W_{+1, T + 1}}{W_{+1, 1}} \leq \frac{\eta}{1 - \gamma} \sum_{t=1}^{T} \sum_{k=1}^K p_{+1, t}(k) \hat{g}_{+1, t}(k) + \frac{\eta^2}{1 - \gamma} \sum_{t=1}^{T} \sum_{k=1}^K p_{+1, t}(k) \hat{g}^2_{+1, t}(k).
\end{equation}
Note that, $W_{+1, 1} = K$ (see Algorithm \ref{alg:3p}). For any action $j \in [K]$,
\begin{equation}
    \label{eq:w_ration_lower_bound_row_player}
    \ln \frac{W_{+1, T + 1}}{W_{+1, 1}} \geq \ln W_{+1, T + 1}(j) - \ln K = \eta\sum_{t=1}^T \hat{g}_{+1, t}(j) - \ln K
\end{equation}
Combining \eqref{eq:w_ration_upper_bound_row_player} and \eqref{eq:w_ration_lower_bound_row_player} gives
\begin{equation*}
    \sum_{t=1}^T \hat{g}_{+1, t}(j) \leq \frac{1}{1 - \gamma} \sum_{t=1}^T \sum_{k=1}^K p_{+1, t}(k) \hat{g}_{+1, t}(k) + \frac{\eta}{1 - \gamma} \sum_{t=1}^T \sum_{k=1}^K p_{+1, t}(k) \hat{g}^2_{+1, t}(k) + \frac{\ln K}{\eta}.
\end{equation*}
Now consider the sum $\sum_{k=1}^K p_{+1, t}(k) \hat{g}_{+1, t}(k)$.
\begin{equation}
    \label{eq:hat_g_expectation}
    \sum_{k=1}^K p_{+1, t}(k) \hat{g}_{+1, t}(k) = \sum_{k=1}^K p_{+1, t}(k) \frac{o_t(k, k_{-1, t})\bfone{k_{+1, t} = k} + \beta}{p_{+1, t}(k)} = o_t(k_{+1, t}, k_{-1, t}) + K\beta.
\end{equation}
As $p_{+1, t}(k) \hat{g}_{+1, t}(k) \leq 1 + \beta$, we get
\begin{align*}
    (1 - \gamma) \sum_{t=1}^T \hat{g}_{+1, t}(j) &\leq \beta KT + \sum_{t=1}^T o_t(k_{+1, t}, k_{-1, t}) + \eta (1 + \beta) \sum_{t = 1}^T \sum_{k = 1}^K \hat{g}_{+1, t}(k) + \frac{\ln K}{\eta} \\
    &\leq \beta KT + \sum_{t=1}^T o_t(k_{+1, t}, k_{-1, t}) + \eta (1 + \beta) K \max_{k \in [K]} \sum_{t = 1}^T \hat{g}_{+1, t}(k) + \frac{\ln K}{\eta} \\
    &\leq \beta KT + \sum_{t=1}^T o_t(k_{+1, t}, k_{-1, t}) + \gamma \max_{k \in [K]} \sum_{t = 1}^T \hat{g}_{+1, t}(k) + \frac{\ln K}{\eta}
\end{align*}
The last inequality uses the fact that $\gamma \geq \eta (1 + \beta) K$ by assumption. As the inequality holds for any arbitrary $j \in [K]$, we can take a $\max$ over $j \in [K]$ on the left hand side. Simplifying this further yields:
\begin{align*}
    (1 - 2\gamma) \max_{j \in [K]} \sum_{t=1}^T \hat{g}_{+1, t}(j) \leq \beta K T + \sum_{t=1}^T o_t(k_{+1, t}, k_{-1, t}) + \frac{\ln K}{\eta}.
\end{align*}
By Azuma-Hoeffding's inequality, with probability at least $1 - \delta/2$, 
\begin{equation*}
    \sum_{t = 1}^T o_t(k_{+1, t}, k_{-1, t}) \leq \sum_{t = 1}^T P_t(k_{+1, t}, k_{-1, t}) + 2\sqrt{2T \ln(2/\delta)}.
\end{equation*}
Combining this with Lemma \ref{lemma:g_hat_upper_bounds_p_row_player}, we get with probability at least $1 - \delta$,
\begin{equation*}
    (1 - 2\gamma) \max_{j \in [K]} \left[ \sum_{t=1}^T P_t(j, k_{-1, t}) - \frac{\ln(2K/\delta)}{\beta} \right] \leq \beta KT + \sum_{t = 1}^T P_t(k_{+1, t}, k_{-1, t}) + 2\sqrt{2T \ln(2/\delta)} + \frac{\ln K}{\eta}.
\end{equation*}
Let $\gamma = 2\eta K \geq \eta (1 + \beta) K$. Then,
\begin{equation*}
    \max_{j \in [K]} \sum_{t=1}^T P_t(j, k_{-1, t}) - \sum_{t = 1}^T P_t(k_{+1, t}, k_{-1, t}) \leq \beta KT + 2\sqrt{2T \ln(2/\delta)} + \frac{\ln K}{\eta} + \frac{\ln(2K/\delta)}{\beta} + 4\eta K T.
\end{equation*}
Setting $\eta = \frac{1}{2}\sqrt{\frac{\ln K}{KT}}$ and $\beta = \sqrt{\frac{\ln K}{KT}}$ finishes the proof.

\end{proof}


\section{Missing details from Dynamic Regret Analysis of 
algs{} (\Cref{section:3s})}
\label{appendix:3s_analysis}


\subsection{Missing Details from Dynamic Regret Analysis in terms of Switching Variation (Section \ref{section:switch_nonstationarity})}
\label{appendix:3s_analysis_switch}

Lemmas \ref{lemma:3s_row_player_bound} and \ref{lemma:3s_column_player_bound} bound $\dreg_T^{+1}(j^T)$ and $\dreg_T^{-1}(j^T)$, respectively.

\begin{restatable}{lemma}{algsrowplayer}
\label{lemma:3s_row_player_bound}
Let $j^T = (j_1, \dots, j_T)$ be an arbitrary sequence of actions with $S$ switches. The following bounds hold with probability at least $1 - \delta$ when $\alpha = \frac{1}{T}$ and $\gamma = 2 \eta K$:
  \begin{align*}
    \dreg_T^{+1}(j^T) &= O \left( S \sqrt{KT} \ln \frac{KT}{\delta} \right) \text{ if } \beta = \eta = \frac{1}{\sqrt{KT}} \\
    \dreg_T^{+1}(j^T) &= O \left( \sqrt{SKT} \ln \frac{KT}{\delta} \right) \text{ if } \beta = \eta = \sqrt{\frac{S}{KT}}.
  \end{align*}
\end{restatable}

\begin{proof}
Assuming that $\eta K \frac{1 + \beta}{\gamma} \leq 1$, a similar calculation as in the proof of Lemma \ref{lemma:3p_bound_row_player} gives
\begin{equation*}
    \ln \frac{W_{+1, t + 1}}{W_{+1, t}} \leq \frac{\eta}{1 - \gamma} \sum_{k=1}^K p_{+1, t}(k) \hat{g}_{+1, t}(k) + \frac{\eta^2}{1 - \gamma} \sum_{k=1}^K p_{+1, t}(k) \hat{g}_{+1, t}^2(k) + e\alpha.
\end{equation*}
Using \eqref{eq:hat_g_expectation} and noting that $p_{+1, t}(k) \hat{g}_{+1, t}(k) \leq 1 + \beta$, we get
\begin{equation*}
    \ln \frac{W_{+1, t + 1}}{W_{+1, t}} \leq \frac{\eta}{1 - \gamma} (o_t(k_{+1, t}, k_{-1, t}) + \beta K) + (1 + \beta) \frac{\eta^2}{1 - \gamma} \sum_{k=1}^K \hat{g}_{+1, t}(k) + e\alpha.
\end{equation*}
Recall that $S = 1 + \abs{\lbrace 1 \leq \ell < T: j_\ell \neq j_{\ell + 1} \rbrace}$. Divide the time horizon $T$ into sub-intervals $$[T_1, \dots, T_2), [T_2, \dots, T_3), \dots, [T_{S}, \dots, T_{S + 1}),$$ where $T_1 = 1$ and $T_{S + 1} = T + 1$ and $j_{T_s} = j_{T_s + 1} = \dots = j_{T_{s + 1} - 1}$ for all sub-intervals $s = 1, \dots, S$. Further, define $\Delta^s = T_{s + 1} - T_s$. Summing over $t \in [T_s, T_{s + 1})$ results in
\begin{equation*}
    \ln \frac{W_{+1, T_{s + 1}}}{W_{+1, T_s}} \leq \frac{\eta}{1 - \gamma} \beta K \Delta^s + \frac{\eta}{1 - \gamma} \sum_{t = T_s}^{T_{s + 1} - 1} o_t(k_{+1, t}, k_{-1, t}) + (1 + \beta) \frac{\eta^2}{1 - \gamma} \sum_{t = T_s}^{T_{s + 1} - 1} \sum_{k = 1}^K \hat{g}_{+1, t}(k) + e\alpha \Delta^s.
\end{equation*}
For any action $j^s \in [K]$,
\begin{align*}
    W_{+1, T_{s + 1}}(j^s) &\geq W_{+1, T_s + 1}(j^s) \exp \left( \eta \sum_{t = T_s + 1}^{T_{s + 1} - 1} \hat{g}_{+1, t}(j^s) \right) \\
    &\geq e \alpha W_{+1, T_s} \exp \left( \eta \sum_{t = T_s + 1}^{T_{s + 1} - 1} \hat{g}_{+1, t}(j^s) \right) \\
    &\geq \alpha W_{+1, T_s} \exp \left( \eta \sum_{t = T_s}^{T_{s + 1} - 1} \hat{g}_{+1, t}(j^s) \right),
\end{align*}
where the last step uses $\eta \hat{g}_{+1, t}(j^s) \leq \eta \frac{1 + \beta}{p_{+1, t}(j^s)} \leq \eta K \frac{1 + \beta}{\gamma} \leq 1$. We have
\begin{equation*}
    \ln \frac{W_{+1, T_{s + 1}}}{W_{+1, T_s}} \geq \ln \frac{W_{+1, T_{s+1}}(j^s)}{W_{+1, T_s}} \geq \ln \alpha + \eta \sum_{t = T_s}^{T_{s + 1} - 1} \hat{g}_{+1, t}(j^s).
\end{equation*}
Combining the upper-bound and lower-bound on $\ln \frac{W_{+1, T_{s + 1}}}{W_{+1, T_s}}$ results in
\begin{equation*}
    \sum_{t = T_s}^{T_{s + 1} - 1} o_t(k_{+1, t}, k_{-1, t}) \geq (1 - \gamma) \sum_{t = T_s}^{T_{s + 1} - 1} \hat{g}_{+1, t}(j^s) - \eta(1 + \beta) \sum_{t = T_s}^{T_{s + 1} - 1} \sum_{k=1}^K \hat{g}_{+1, t}(k) - \left(\frac{e \alpha }{\eta} - \beta K \right) \Delta^s - \frac{1}{\eta} \ln 1/\alpha. 
\end{equation*}
Using Azuma-Hoeffding inequality, with probability at least $1 - \delta/2$,
\begin{align*}
    \sum_{t = T_s}^{T_{s + 1} - 1} &P_t(k_{+1, t}, k_{-1, t}) + 2\sqrt{2 \Delta^s \ln(2/\delta)} \geq \sum_{t = T_s}^{T_{s + 1} - 1} o_t(k_{+1, t}, k_{-1, t}) \\
    &\geq (1 - \gamma) \sum_{t = T_s}^{T_{s + 1} - 1} \hat{g}_{+1, t}(j^s) - \eta(1 + \beta) \sum_{t = T_s}^{T_{s + 1} - 1} \sum_{k=1}^K \hat{g}_{+1, t}(k) - \left(\frac{e \alpha }{\eta} - \beta K \right) \Delta^s - \frac{1}{\eta} \ln 1/\alpha \\
    &\geq (1 - \gamma) \sum_{t = T_s}^{T_{s + 1} - 1} \hat{g}_{+1, t}(j^s) - \gamma \max_{k \in [K]} \sum_{t = T_s}^{T_{s + 1} - 1} \hat{g}_{+1, t}(k) - \left(\frac{e \alpha }{\eta} - \beta K \right) \Delta^s - \frac{1}{\eta} \ln 1/\alpha \\
    &\geq (1 - 2\gamma) \max_{k \in [K]} \sum_{t = T_s}^{T_{s + 1} - 1} \hat{g}_{+1, t}(k) - \left(\frac{e \alpha }{\eta} - \beta K \right) \Delta^s - \frac{1}{\eta} \ln 1/\alpha.
\end{align*}
The last inequality holds as the expression above it is valid for all choices of $j^s \in [K]$, and hence we can take a maximum over all such choices. A calculation similar to the one in the proof of Lemma \ref{lemma:g_hat_upper_bounds_p_row_player} implies that $\sum_{t = T_s}^{T_{s + 1} - 1} \hat{g}_{+1, t}(k) \geq \sum_{t=T_s}^{T_{s+1} - 1} P_t(k, k_{-1, t}) - \frac{\ln(2K/\delta)}{\beta}$ for all $k \in [K]$ with probability at least $1 - \delta/2$. Therefore, with probability at least $1 - \delta$,

\begin{equation}
\label{eq:exp3_t}
    \sum_{t = T_s}^{T_{s + 1} - 1} P_t(k_{+1, t}, k_{-1, t}) + 2\sqrt{2 \Delta^s \ln(2/\delta)} \geq (1 - 2\gamma) \max_{k \in [K]} \sum_{t = T_s}^{T_{s + 1} - 1} P_t(k, k_{-1, t}) - \left(\frac{e \alpha }{\eta} - \beta K \right) \Delta^s - \frac{1}{\eta} \ln \frac{1}{\alpha} - \frac{\ln(2K/\delta)}{\beta}.
\end{equation}

Setting $\gamma = 2 \eta K \geq \eta (1 + \beta) K$, rearranging, and summing over $s = 1, \dots, S$, we get
\begin{align*}
    \sum_{s = 1}^S \max_{k \in [K]} \sum_{t = T_s}^{T_{s + 1} - 1} P_t(k, k_{-1, t}) - \sum_{t=1}^T P_t(k_{+1, t}, k_{-1, t}) \leq &\sum_{s=1}^S 2\sqrt{2\Delta^s \ln(2/\delta)} + \left(\frac{e \alpha }{\eta} - \beta K \right) T + \frac{S}{\eta} \ln \frac{1}{\alpha} \\
    &+ \frac{S}{\beta} \ln \frac{2K}{\delta} + 4\eta KT.
\end{align*}
In particular, we can use the sequence $j^T$ on the left side above as this sequence uses a constant action within intervals $[T_s, T_{s+1})$.
\begin{align*}
    \sum_{t = 1}^{T} P_t(j_t, k_{-1, t}) - \sum_{t=1}^T P_t(k_{+1, t}, k_{-1, t}) &\leq \sum_{s=1}^S 2\sqrt{2\Delta^s \ln(2/\delta)} + \left(\frac{e \alpha }{\eta} - \beta K \right) T + \frac{S}{\eta} \ln \frac{1}{\alpha} + \frac{S}{\beta} \ln \frac{2K}{\delta} + 4\eta KT 
\end{align*}

Using $\alpha = \frac{1}{T}$ and $\eta = \beta = \sqrt{\frac{1}{KT}}$ results in:
\begin{align*}
    \sum_{t = 1}^{T} P_t(j_t, k_{-1, t}) - \sum_{t=1}^T P_t(k_{+1, t}, k_{-1, t}) &\leq 2\sqrt{2ST\ln(2/\delta)} + (e-1)\sqrt{KT} + S \sqrt{KT} \ln \frac{2KT}{\delta} + 4 \sqrt{KT} \\
    &= O\left(S\sqrt{KT} \ln \frac{KT}{\delta} \right).
\end{align*}
On the other hand, if $\alpha = \frac{1}{T}$ and $\beta = \eta = \sqrt{\frac{S}{KT}}$ (i.e., $S$ is used for setting $\beta$ and $\eta$), we get
\begin{align*}
    \sum_{t = 1}^{T} P_t(j_t, k_{-1, t}) - \sum_{t=1}^T P_t(k_{+1, t}, k_{-1, t}) &\leq 2\sqrt{2ST\ln(2/\delta)} + e \sqrt{\frac{KT}{S}} - \sqrt{SKT} + \sqrt{SKT} \ln \frac{2KT}{\delta} + 4 \sqrt{SKT} \\
    &= O\left(\sqrt{SKT} \ln \frac{KT}{\delta} \right),
\end{align*}
which concludes the claim.
\end{proof}

\begin{restatable}{lemma}{algscolumnplayer}
  \label{lemma:3s_column_player_bound}
  Assume the same setup as Lemma \ref{lemma:3s_row_player_bound}. With probability at least $1 - \delta$, when $\alpha = \frac{1}{T}$ and $\gamma = 2 \eta K$:
  \begin{align*}
    \dreg_T^{-1}(j^T) &= O \left( S \sqrt{KT} \ln \frac{KT}{\delta} \right) \text{ if } \beta = \eta = \frac{1}{\sqrt{KT}} \\
    \dreg_T^{-1}(j^T) &= O \left( \sqrt{SKT} \ln \frac{KT}{\delta} \right) \text{ if } \beta = \eta = \sqrt{\frac{S}{KT}}.
  \end{align*}
\end{restatable}

\begin{proof}
  The proof uses similar arguments as Lemma \ref{lemma:3s_row_player_bound} and has been omitted.
\end{proof}


\subsection{Missing Details of Dynamic Regret Analysis in Terms of Continuous Variation (Section \ref{section:continuous_nonstationarity})}
\label{app:cont_nst}

\cvalgsanalysis*

\begin{proof}
Lem. \ref{lemma:3scv_row_player_bound} builds the key result which shows for any arbitrary sequence of preference matrices $\P_1,\ldots,\P_T$ with \cv\, $V_T$, the dynamic regret is upper bounded by $\tO(V_T^{1/3}K^{1/3}T^{2/3})$.
The proof now follows recalling that by construction  
$\dreg_T(j^T):= \dreg_T^{+1}(j^T) + \dreg_T^{-1}(j^T)$ and $\dreg_T(\algs) = \max_{j^T \in [K]^T} \dreg_T(j^T)$. 
\end{proof}

\begin{restatable}{lemma}{algsrowplayercv}
\label{lemma:3scv_row_player_bound}
Let $j^T = (j_1, \ldots, j_T)$ be an arbitrary sequence of $T$ actions and $\P_1,\ldots,\P_T$ be the underlying sequence of preference matrices with \cv\, $V_T$. Then setting $\alpha = \frac{1}{T}$, $\gamma = 2 \eta K$, $\beta = \eta = \frac{V_T^{1/3}}{4K^{2/3}T^{1/3}}$, with probability at least $1 - \delta$:
  \begin{align*}
    \dreg_T^{+1}(j^T) &= O\Big( \big(V_T^{1/3}K^{1/3}T^{2/3} + 4K^{2/3}T^{1/3}V_T^{-1/3} \big)\ln \frac{KT}{\delta}\Big)\\
    \dreg_T^{-1}(j^T) &= O\Big( \big(V_T^{1/3}K^{1/3}T^{2/3} + 4K^{2/3}T^{1/3}V_T^{-1/3} \big)\ln \frac{KT}{\delta}\Big).
  \end{align*}
\end{restatable}

\begin{proof}
Drawing ideas from the proof analyses of \cite{besbes+14}, let us first divide the time interval $[T]$ into sub-intervals $T_1,\ldots, T_{S}$, each of size at most $\Delta$ (hence $S:=\lceil T/\Delta \rceil$), such that 
\[T_j:=\{t \mid (j-1)\Delta +1 \le t \le \min\{j\Delta,T\}\}.\]

Recall at each time $t$ we may break the instantaneous regret w.r.t. any arm $j_t \in [K]$ as, 
$
    r_t(j_t):= [P_t(j_t, k_{-1, t}) + P_t(j_t, k_{+1, t})] = r_t^{+1}(j_t) + r_t^{-1}(j_t).
$

\begin{align*}
    r_t^{+1}(j_t) &= [P_t(j_t, k_{-1, t}) - P_t(k_{+1, t}, k_{-1, t})] \\
    r_t^{-1}(j_t) &= [P_t(j_t, k_{+1, t}) - P_t(k_{-1, t}, k_{+1, t})]
\end{align*}

For any phase $i \in [\frac{T}{\Delta}]$, note we can break the total regret of the left-arm at sub-interval $i$ w.r.t any arm-sequence $\{j_t\}_{t \in \cT_i}$ as:  
\begin{align*}
R_{\cT_i}^{+1}(\{j_t\}_{t \in \cT_i})&:= \sum_{t \in \cT_i}[P_t(j_t, k_{-1, t}) - P_t(k_{+1, t}, k_{-1, t})]\\
& = \sum_{t \in \cT_i}[P_t(j_t, k_{-1, t}) - P_t(i^*, k_{-1, t}) + P_t(i^*, k_{-1, t}) - P_t(k_{+1, t}, k_{-1, t})]\\
& = \sum_{t \in \cT_i}[P_t(j_t, k_{-1, t}) - P_t(i^*, k_{-1, t})] + \sum_{t \in \cT_i}[P_t(i^*, k_{-1, t}) - P_t(k_{+1, t}, k_{-1, t})]\\
& \le 2V_{\cT_i}\Delta + \sum_{t \in \cT_i}[P_t(i^*, k_{-1, t}) - P_t(k_{+1, t}, k_{-1, t})]\\
& \le 2V_{\cT_i}\Delta +  \bigg[2\sqrt{2|\cT_i| \ln(2/\delta)} + \left(\frac{e \alpha }{\eta} - \beta K \right) |\cT_i| + \frac{1}{\eta} \ln \frac{1}{\alpha} + \frac{1}{\beta} \ln \frac{2K}{\delta} + 4\eta K|\cT_i| \bigg],
\end{align*}
here $i^*$ is any representative arm of sub-interval $i$.
Then summing over all sub-intervals  $\cT_1, \ldots, \cT_S$, the total regret of the left-arm w.r.t any arm-sequence $j^T$ is: 

\begin{align*}
\dreg_T^{+1}(j^T)&:= \sum_{s = 1}^S R_{\cT_i}^{+1}(\{j_t\}_{t \in \cT_i}) \\
& \le \sum_{s = 1}^S \bigg[ 2V_{\cT_i}\Delta + \big[2\sqrt{2|\cT_i| \ln(2/\delta)} + \left(\frac{e \alpha }{\eta} - \beta K \right) |\cT_i| + \frac{1}{\eta} \ln \frac{1}{\alpha} + \frac{1}{\beta} \ln \frac{2K}{\delta} + 4\eta K|\cT_i|\big] \bigg]\\
& = 2V_T \Delta + 2\sqrt{2 T \ln(2/\delta)} + \left(\frac{e\alpha T}{\eta} - \beta K T \right) + \frac{T}{\eta \Delta} \ln \frac{1}{\alpha}  + \frac{1}{\eta} \ln \frac{1}{\alpha} + \frac{T}{\beta \Delta} \ln \frac{2K}{\delta} + \frac{1}{\beta } \ln \frac{2K}{\delta} + 8\eta K T\\
& \overset{(a)}{=} 2V_T \Delta + 2\sqrt{2 T \ln(2/\delta)} + \left(\frac{e}{\eta} - \sqrt{K T} \right) + \frac{T}{\eta \Delta} \ln T  + \frac{1}{\eta} \ln T + \frac{T}{\beta \Delta} \ln \frac{2K}{\delta} + \frac{1}{\beta} \ln \frac{2K}{\delta} + 8\eta K T \\
& \overset{(b)}{=} 6 V_T^{1/3}K^{1/3}T^{2/3} + 2\sqrt{2}V_T^{1/3}K^{1/3}T^{2/3}\big( \ln T + \ln (2K/\delta) \big) + 4K^{2/3}T^{1/3}V_T^{-1/3}\big( \ln T + \ln (2K/\delta) \big)\\ 
& = O\Big( \big(V_T^{1/3}K^{1/3}T^{2/3} + 4K^{2/3}T^{1/3}V_T^{-1/3} \big)\big( \ln T + \ln (2K/\delta) \big) \Big),
\end{align*}
where $(a)$ holds by setting $\alpha = 1/T$, $(b)$ holds by setting $\Delta = \frac{T^{2/3}K^{1/3}\sqrt 2}{V^{2/3}}$, and $\beta = \eta = \frac{V_T^{1/3}}{4K^{2/3}T^{1/3}}$.

Following a similar analysis also yields

\begin{align*}
\dreg_T^{-1}(j^T) &= O\Big( \big(V_T^{1/3}K^{1/3}T^{2/3} + 4K^{2/3}T^{1/3}V_T^{-1/3} \big)\big( \ln T + \ln (2K/\delta) \big) \Big),
\end{align*}
combining which with the earlier concludes the proof. 
\end{proof}


\section{Missing details from the Lower Bound (\Cref{section:lower_bound}) }
\label{appendix:lower_bound}


\subsection{Proof of Theorem \ref{theorem:lower_bound_switch}}
\label{appendix:lower_bound_switch}

In what follows, we consider only deterministic algorithms. An execution of a randomized algorithm corresponds to a particular choice from the set of all deterministic algorithms. As the changes in the environment are not affected by the algorithm's decisions, it is enough to show the lower bound for all deterministic algorithms \citep{auer+02N}.

Let $H_t = \lbrace (k_{+1, t'}, k_{-1, t'}, o_{t'}) \rbrace_{t{'} \leq t}$ be the history of observations upto time $t$, where $o_{t'} \coloneqq o_{t'}(k_{+1, t{'}}, k_{-1, t{'}})$ is the feedback returned by the environment at time $t{'}$. We use $\mathrm{P}^s_j$ to denote the probability distribution over histories induced by assuming that arm $j \in [K] \backslash \lbrace 1 \rbrace$ is the condorcet winner in the sub-interval $[T_s, T_{s + 1})$. Similarly, $\mathrm{P}^s_1$ is the probability distribution assuming that arm $1$ is the condorcet winner in the sub-interval $[T_s, T_{s + 1})$ with probability of winning $0.5 + \epsilon$ against all arms and $P_t(i, j) = 0.5$ for all $i, j \in [K] \backslash \lbrace 1 \rbrace$. The corresponding expectations are denoted by $\mathrm{E}_j^s \left[ \cdot \right]$ and $\mathrm{E}_1^s \left[ \cdot \right]$, respectively. 

The following technical lemma assumes a simpler case where $P_t(i, j)$ remains constant for all $t \in [T]$. One can think of this as the case when $S = 1$. We omit the superscript in $\mathrm{P}^s_j$, $\mathrm{P}^s_1$, $\mathrm{E}_j^s \left[ \cdot \right]$, $\mathrm{E}_1^s \left[ \cdot \right]$, and $j^s$ in the lemma below to simplify the notation.

\begin{lemma}
  \label{lemma:expected_values_under_different_distributions}
  Let $\calH_T$ be the set of all possible histories $H_T$ and let $f: \calH_T \rightarrow [0, M]$ be a measurable function that maps a history $H_T$ to a number in the interval $[0, M]$. Then,
  $$\mathrm{E}_j \left[ f(H_T) \right] \leq \mathrm{E}_1 \left[ f(H_T) \right] + M \sqrt{ \epsilon \ln \left( \frac{1 + 2\epsilon}{1 - 2\epsilon} \right) \mathrm{E}_1 \left[ N_{1j} + N_j \right]},$$ where, $N_{1j}$ is the number of times the algorithm chooses to duel between arms $1$ and $j$ and $N_j$ is the number of times arm $j$ is compared with an arm other than itself and arm $1$.
\end{lemma}

\begin{proof}

For any two probability distributions $p$ and $q$, let $\rmTV{p}{q}$ and $\rmKL{p}{q}$, respectively, denote the total variation distance and KL-divergence between $p$ and $q$. Note that,
\begin{align}
  \label{eq:EjE1_diff}
  \mathrm{E}_j \left[ f(H_T) \right] - \mathrm{E}_1 \left[ f(H_T) \right] &= \sum_{H_T \in \calH_T} f(H_T) \left( \mathrm{P}_j(H_T) - \mathrm{P}_1(H_T) \right) \notag \\
  &\leq \sum_{H_T: \mathrm{P}_j(H_T) \geq \mathrm{P}_1(H_T)} f(H_T) \left( \mathrm{P}_j(H_T) - \mathrm{P}_1(H_T) \right) \notag \\
  &\leq M \sum_{H_T: \mathrm{P}_j(H_T) \geq \mathrm{P}_1(H_T)} \left( \mathrm{P}_j(H_T) - \mathrm{P}_1(H_T) \right) \notag \\
  &\overset{(a)}{=} M \;\; \rmTV{\mathrm{P}_j}{\mathrm{P}_1} \notag \\ 
  &\overset{(b)}{\leq} M \sqrt{\frac{\ln 2}{2} \; \rmKL{\mathrm{P}_1}{\mathrm{P}_j}}.
\end{align}
Here, $(a)$ follows from the definition of the total-variation distance and $(b)$ follows from Pinsker's inequality. Let $h_t = (k_{+1, t}, k_{-1, t}, o_t)$ be the $t^{th}$ element of $H_T$. Using the chain-rule for KL-divergence, we get
\begin{align*}
  \rmKL{\mathrm{P}_1}{\mathrm{P}_j} &= \sum_{t = 1}^T \rmKL{\mathrm{P}_1\left( h_t \vert H_{t - 1} \right)}{\mathrm{P}_j\left( h_t \vert H_{t - 1} \right)} \\
  &= \sum_{t = 1}^T \mathrm{P}_1((k_{+1, t} = k_{-1, t}) \lor (k_{+1, t} \notin \lbrace 1, j \rbrace \land k_{-1, t} \notin \lbrace 1, j \rbrace)) \;\; \rmKL{\frac{1}{2}}{\frac{1}{2}} + \\
  &\hspace{1.2cm} \mathrm{P}_1( (k_{+1, t} = 1 \land k_{-1, t} = j) \lor (k_{+1, t} = j \land k_{-1, t} = 1) ) \;\; 2\epsilon \log_2 \left( \frac{1 + 2\epsilon}{1 - 2\epsilon} \right) - \\
  &\hspace{1.2cm} \mathrm{P}_1( (k_{+1, t} = j \land k_{-1, t} \notin \lbrace 1, j \rbrace) \lor (k_{+1, t} \notin \lbrace 1, j \rbrace \land k_{-1, t} = j) ) \;\; \frac{1}{2} \log_2 \left( 1 - 4\epsilon^2 \right) \\
  &= 2 \mathrm{E}_1[N_{1j}] \epsilon \log_2 \left( \frac{1 + 2\epsilon}{1 - 2\epsilon} \right) - \frac{1}{2} \mathrm{E}_1[N_j] \log_2 \left( 1 - 4\epsilon^2 \right) \\
  &\leq 2 \epsilon \log_2 \left( \frac{1 + 2\epsilon}{1 - 2\epsilon} \right) \mathrm{E}_1 \left[ N_{1j} + N_j \right].
\end{align*}
The last line follows from the fact that $-\frac{1}{2} \log_2 (1 - 4x^2) \leq 2 x\log_2 \frac{1 + 2x}{1 - 2x}$ for all $x \in (0, 0.5)$. We have abused the notation and used $\rmKL{\alpha}{\beta}$ to represent $\rmKL{p}{q}$ where $p = \mathrm{Bernoulli}(\alpha)$ and $q = \mathrm{Bernoulli}(\beta)$ for $\alpha, \beta \in [0, 1]$. Combining this inequality with \eqref{eq:EjE1_diff} yields the desired result.
\end{proof}

Theorem \ref{theorem:lower_bound_switch} has been reproduced verbatim from Section \ref{section:lower_bound} below. See Section \ref{section:lower_bound} for the definition of $\calV$ and $j^T(\nu)$.
\lowerboundswitch*

\begin{proof}

Let $t \in [T_s, T_{s + 1})$ and define the following events at time $t$:
\begin{enumerate}
    \item $\calE_a^t = \lbrace \left( k_{+1, t} = 1 \land k_{-1, t} = j^s \right) \lor \left(k_{+1, t} = j^s \land k_{-1, t} = 1 \right) \rbrace$
    \item $\calE_b^t = \lbrace \left( k_{+1, t} = j^s \land k_{-1, t} \notin \lbrace 1, j^s \rbrace \right) \lor \left(k_{+1, t} \notin \lbrace 1, j^s \rbrace \land k_{-1, t} = j^s \right) \rbrace$
    \item $\calE_c^t = \lbrace k_{+1, t} \neq j^s \land k_{-1, t} \neq j^s \rbrace$
    \item $\calE_d^t = \lbrace k_{+1, t} = j^s \land k_{-1, t} = j^s \rbrace$
\end{enumerate}
For $t \in [T_s, T_{s + 1})$, let $r_t = P_t(j^s, k_{+1, t}) + P_t(j^s, k_{-1, t}) - 1) / 2$. Now consider the expected regret at time $t \in [T_s, T_{s + 1})$ under the distribution $\mathrm{P}_{j^s}^s$. 
\begin{align*}
  \mathrm{E}_{j^s}^s \left[ r_t \right] &= \frac{\epsilon}{2} \; \mathrm{P}_{j^s}^s \left( \calE^t_a \right) + \frac{\epsilon}{2} \; \mathrm{P}_{j^s}^s \left( \calE_b^t \right) + \epsilon \; \mathrm{P}_{j^s}^s \left( \calE_c^t \right) \\
  &= \frac{\epsilon}{2} \; \mathrm{P}_{j^s}^s \left( \calE^t_a \right) + \frac{\epsilon}{2} \; \mathrm{P}_{j^s}^s \left( \calE_b^t \right) + \epsilon \; \left( 1 - \mathrm{P}_{j^s}^s \left(\calE_a^t \right) - \mathrm{P}_{j^s}^s \left(\calE_b^t \right) - \mathrm{P}_{j^s}^s \left(\calE_d^t \right) \right) \\
  &= \epsilon - \epsilon \mathrm{P}_{j^s}^s \left( \calE^t_d \right) - \frac{\epsilon}{2} \left( \mathrm{P}_{j^s}^s \left( \calE^t_a \right) +  \mathrm{P}_{j^s}^s \left( \calE^t_b \right) \right).
\end{align*}
Define $N_{1j^s}^s = \sum_{t = T_s}^{T_{s + 1} - 1} \bfone{\calE_a^t}$, $N_{j^s}^s = \sum_{t = T_s}^{T_{s + 1} - 1} \bfone{\calE_b^t}$, $\bar{N}_{j^s}^s = \sum_{t = T_s}^{T_{s + 1} - 1} \bfone{\calE_d^t}$, and $\Delta^s = T_{s + 1} - T_s$. We have,
\begin{align*}
    \mathrm{E}_{j^s}^s \left[ \sum_{t = T_s}^{T_{s + 1} - 1} r_t \right] &= \epsilon \Delta^s - \epsilon \mathrm{E}_{j^s}^s \left[ \bar{N}^s_{j^s} \right] - \frac{\epsilon}{2} \left( \mathrm{E}_{j^s}^s \left[ N_{1j^s}^s \right] + \mathrm{E}_{j^s}^s \left[ N_{j^s}^s \right] \right) \\
    &= \epsilon \Delta^s - \frac{\epsilon}{2} \mathrm{E}_{j^s}^s \left[ 2 \bar{N}^s_{j^s} + N_{1j^s}^s + N_{j^s}^s \right].
\end{align*}
We bound the expectation on the right hand side above using Lemma \ref{lemma:expected_values_under_different_distributions}. As the learner only uses the history $H_t$ to choose the arms at time $t$, $\bar{N}_{j^s}^s$, $N_{1j^s}^s $, $N_{j^s}^s$, and hence $2 \bar{N}^s_{j^s} + N_{1j^s}^s + N_{j^s}^s$ is a measurable function of the history. Applying Lemma \ref{lemma:expected_values_under_different_distributions} and noting that the maximum value of $2 \bar{N}^s_{j^s} + N_{1j^s}^s + N_{j^s}^s$ is $2 \Delta^s$ results in
\begin{align*}
    \mathrm{E}_{j^s}^s \left[ 2 \bar{N}^s_{j^s} + N_{1j^s}^s + N_{j^s}^s \right] \leq \mathrm{E}_1^s \left[ 2 \bar{N}^s_{j^s} + N_{1j^s}^s + N_{j^s}^s \right] + 2 \Delta^s \sqrt{ \epsilon \ln \left( \frac{1 + 2\epsilon}{1 - 2\epsilon} \right) \mathrm{E}_1 \left[ N_{1j}^s + N_j^s \right]}.
\end{align*}
Therefore, using Cauchy-Schwarz inequality,
\begin{align*}
    \sum_{j^s = 2}^K \mathrm{E}_{j^s}^s \left[ 2 \bar{N}^s_{j^s} + N_{1j^s}^s + N_{j^s}^s \right] &\leq \mathrm{E}_1^s \left[ \sum_{j^s = 2}^K 2 \bar{N}^s_{j^s} + N_{1j^s}^s + N_{j^s}^s \right] + 2 \Delta^s \sqrt{ \epsilon \ln \left( \frac{1 + 2\epsilon}{1 - 2\epsilon} \right) (K - 1) \sum_{j^s = 2}^K \mathrm{E}_1 \left[ N_{1j}^s + N_j^s \right]} \\
    &\leq 2 \Delta^s + 2 \Delta^s \sqrt{2\Delta^s (K - 1) \epsilon \ln \left( \frac{1 + 2\epsilon}{1 - 2\epsilon} \right)}.
\end{align*}
This results in the following bound:
\begin{align*}
    \sum_{j^s = 2}^K \mathrm{E}_{j^s}^s \left[ \sum_{t = T_s}^{T_{s + 1} - 1} r_t \right] &\geq \epsilon \Delta^s \left(K - 1 - 1 -  \sqrt{2\Delta^s (K - 1) \epsilon \ln \left( \frac{1 + 2\epsilon}{1 - 2\epsilon} \right)} \right) \\
    &\geq \epsilon \Delta^s \left(K - 1 - 1 - 8\epsilon \sqrt{\Delta^s (K - 1) \ln(4/3)} \right).
\end{align*}
The last line uses the inequality $2x \ln \frac{1 + 2x}{1 - 2x} \leq 32  x^2 \ln (4/3)$ for all $x \in [0, 1/4]$.

Recall that $\dreg_T(j^T(\nu)) = \sum_{s = 1}^{S} \sum_{t = T_s}^{T_{s + 1} - 1} r_t$. The expected value of $\dreg_T(j^T(\nu))$ under a randomly sampled environment $\nu$ from $\calV$ is given by:
\begin{align*}
    \rmE{\dreg_T(j^T(\nu))} &= \sum_{s = 1}^{S} \frac{1}{K - 1} \sum_{j^s = 2}^K \mathrm{E}_{j^s}^s \left[ \sum_{t = T_s}^{T_{s + 1} - 1} r_t \right] \\
    & \geq \sum_{s = 1}^{S} \left( \epsilon \Delta^s - \frac{\epsilon}{K - 1} \Delta^s - \frac{8\epsilon^2}{K - 1} \Delta^s \sqrt{\Delta^s (K - 1) \ln (4/3)} \right) \\
    & \geq \sum_{s = 1}^{S} \left( \epsilon \Delta^s - \frac{\epsilon}{K - 1} \Delta^s - \frac{8\epsilon^2}{K - 1} \Delta^s \sqrt{\frac{T(K - 1)}{S} \ln (4/3)} \right) \\
    &= \epsilon T - \epsilon \frac{T}{K - 1} - 8 \epsilon^2 \sqrt{\frac{T(K - 1)}{S} \ln (4/3)} \;\; \frac{T}{K - 1} \\
    &= \epsilon \frac{K - 2}{K - 1} T - 8 \epsilon^2 \sqrt{\frac{T(K - 1)}{S} \ln (4/3)} \;\; \frac{T}{K - 1} \\
    &\overset{(a)}{\geq} \epsilon \frac{T}{2} - 8 \epsilon^2 \sqrt{\frac{T(K - 1)}{S} \ln (4/3)} \;\; \frac{T}{K - 1}.
\end{align*}
Here, $(a)$ follows by assuming that $K \geq 3$. Choosing the value of $\epsilon = \frac{1}{2} \sqrt{\frac{S(K - 1)}{T \ln (4/3)}}$ to maximize the expression on the right hand side results in
\begin{align*}
  \rmE{\dreg_T(j^T(\nu))} &\geq \frac{1}{4\sqrt{\ln(4/3)}} \sqrt{S(K - 1)T} - \frac{2}{\sqrt{\ln(4/3)}} \sqrt{S(K - 1)T} \\
    &= \Omega\left( \sqrt{SKT} \right).
\end{align*}
As the expected value of $\dreg_T(j^T(\nu))$ under a random choice of environment $\nu \in \calV$ is $\Omega\left( \sqrt{SKT} \right)$, there is at least one environment in $\calV$ for which $\rmE{\dreg_T(j^T(\nu))} = \Omega\left( \sqrt{SKT} \right)$.
\end{proof}


\subsection{Proof of Theorem \ref{theorem:lower_bound_continuous}}
\label{appendix:lower_bound_continuous}

As before, the theorem below has been reproduced verbatim from Section \ref{section:lower_bound}. Please refer to Section \ref{section:lower_bound} for the definition of $\calV$ and $j^T(\nu)$.

\lowerboundcontinuous*

\begin{proof}

Let us define $\calV{'} \subseteq \calV$ such that for all $\nu \in \calV{'}$, $\Delta^1 = \Delta^2 = \dots = \Delta^{S - 1} = \Delta$ and $\Delta^S < \Delta$, where recall that $\Delta^s = T_{s + 1} - T_s$. As $\max_{i, j \in [K]} \abs{P_{t + 1}(i, j) - P_t(i, j)} \leq 2\epsilon$ for all $t > 0$, for any environment $\nu \in \calV{'}$,
\begin{equation*}
    \sum_{t = 1}^{T - 1} \max_{i, j \in [K]} \abs{P_{t + 1}(i, j) - P_t(i, j)} = \sum_{s = 2}^{S} \max_{i, j \in [K]} \abs{P_{T_s}(i, j) - P_{T_s - 1}(i, j)} \leq 2 \epsilon \left\lceil \frac{T}{\Delta} \right\rceil \leq V_T.
\end{equation*}
The last inequality above assumes that $\epsilon \leq \frac{V_T}{2\lceil T/\Delta \rceil}$.  

Proceeding as in the proof of Theorem \ref{theorem:lower_bound_switch}, and noting that $S = \left \lceil \frac{T}{\Delta} \right\rceil$ we get:
\begin{equation*}
    \rmE{\dreg_T(j^T(\nu))} \geq \epsilon T - \frac{\epsilon}{K - 1} T - \frac{8\epsilon^2}{K - 1} T \sqrt{\frac{T (K - 1)}{\lceil T/ \Delta \rceil} \ln (4/3)},
\end{equation*}
where the expectation this time includes a random choice of environment from $\calV{'}$. Further assuming that $\epsilon = \min \left\{ \frac{V_T}{2 \lceil T / \Delta \rceil}, \frac{1}{16\sqrt{\Delta K \ln(4/3)}} \right\}$ results in
\begin{align*}
    \rmE{\dreg_T(j^T(\nu))} &\geq \epsilon \frac{K - 2}{K - 1} T - \frac{8\epsilon^2}{K - 1} T \sqrt{\frac{T (K - 1)}{\lceil T/ \Delta \rceil} \ln (4/3)} \\
    &\geq \frac{\epsilon}{2} T - \frac{8\epsilon^2}{2} T \sqrt{\Delta (K - 1) \ln (4/3)} \\
    &\geq \epsilon T \left(\frac{1}{2} - 4 \epsilon \sqrt{\Delta K \ln (4/3)} \right) \geq \frac{T}{64 \sqrt{\Delta K \ln(4/3)}}.
\end{align*}
Setting $\Delta = \frac{1}{K^{5/3}} \left( \frac{T}{V_T} \right)^{2/3}$ yields the desired result.

\end{proof}


\section{Missing details from Dynamic Regret Analysis under Borda Scores (\Cref{sec:ub_borda})}
\label{appendix:analysis_with_borda_scores}

Algorithm \ref{alg:dexp3_whp} details the procedure for \algdexp{} and defines the reward estimate $s'_{i, t}(k)$ for $i \in \{+1, -1\}$ and $k \in [K]$ in L9. Before we prove the high-probability dynamic regret guarantee from Theorem \ref{theorem:3s_borda_bound}, we reproduce Lemma 10 from \citet{ADB}, which serves the same purpose as \cref{lemma:g_hat_upper_bounds_p_row_player} for the Borda score estimate $s'_{i, t}(k)$ defined in Algorithm \ref{alg:dexp3_whp}.

\begin{lemma}
    \label{lemma:s_bound_borda_score}
    For any $\delta > 0$, $\beta, \gamma \in (0, 1)$, $i \in \{+1, -1\}$, with probability at least $1 - \delta/2$,
    \begin{align*}
        \forall k \in [K], \;\;\;\; \sum_{t=1}^T s'_{i, t}(k) \geq \sum_{t=1}^T s_t(k) - \frac{1}{\gamma\beta} \ln\left( \frac{2K}{\delta}\right),
    \end{align*}
    where $s_t(k) = \frac{1}{K} \sum_{j \in [K]} P_t(k, j)$ is the shifted Borda score of arm $k$ at time $t$.
\end{lemma}

The basic regret decomposition idea that was used in \cref{appendix:3s_analysis} to bound the regret of \algs{} can be used for \algdexp{} as well. We define $\dbreg^+_T(j^T)$ and $\dbreg^-_T(j^T)$ as:
\begin{align*}
    \dbreg^+_T(j^T) &= \sum_{t = 1}^T b_t(j_t) - b_t(k_{+1, t}) = \frac{K}{K - 1} \sum_{t = 1}^T s_t(j_t) - s_t(k_{+1, t}) \\
    \dbreg^-_T(j^T) &= \sum_{t = 1}^T b_t(j_t) - b_t(k_{-1, t}) = \frac{K}{K - 1} \sum_{t = 1}^T s_t(j_t) - s_t(k_{-1, t}).
\end{align*}
As before, $\dbreg_T(j^T) = \frac{1}{2}\left(\dbreg^+_T(j^T) + \dbreg^-_T(j^T) \right)$. We only show the bound on $\dbreg^+_T(j^T)$ below. The bound on $\dbreg^-_T(j^T)$, and hence on $\dbreg_T(j^T)$, follows along the lines of \cref{appendix:3s_analysis}. We begin with a simple lemma.

\begin{lemma}
    \label{lemma:sum_psprime_high_prob_bound}
    With probability at least $1 - \delta/2$,
    \begin{align*}
        \sum_{t=1}^T \sum_{k=1}^K p_{+1, t}(k) s'_{+1, t}(k) \leq \sum_{t=1}^T \left(s_t(k_{+1, t}) + \beta K \right) + \frac{\gamma + 1}{\gamma} \sqrt{2T \ln (2/\delta)}.
    \end{align*}
\end{lemma}

\begin{proof}
    Let $X_n = \sum_{t=1}^n \left[ \sum_{k=1}^K p_{+1, t}(k) s'_{+1, t}(k) - \left(s_t(k_{+1, t}) + \beta K \right)  \right]$. One can easily show that the sequence $X_1, X_2, \dots$ is a martingale. Using Azuma's inequality,
    \begin{equation}
        \label{eq:azuma_martingale}
        \rmP{X_T \geq \epsilon} \leq \exp \left( -\frac{\epsilon^2}{2 \sum_{t=1}^T c_t^2} \right),
    \end{equation}
    where $c_t = b_t - a_t$ where $a_t \leq X_t - X_{t-1} \leq b_t$ almost surely. We next derive such a $c_t$.
    \begin{align*}
        X_t - X_{t-1} &= \sum_{k=1}^K p_{+1, t}(k) s'_{+1, t}(k) - \left(s_t(k_{+1, t}) + \beta K \right) \\
        &= \sum_{k=1}^K \left[ \frac{\bbI\{ k_{+1, t} = k \}}{K} \sum_{j=1}^K \frac{\bbI\{k_{-1, t} = j\} o_t(k, j)}{p_{-1, t}(j)} + \beta \right] - \frac{1}{K}\sum_{j=1}^K P_t(k_{+1, t}, j) - \beta K \\
        &\leq \frac{1}{\gamma} + \beta K - \beta K = \frac{1}{\gamma} = b_t.
    \end{align*}
    Similarly,
    \begin{align*}
        X_t - X_{t - 1} &= \sum_{k=1}^K p_{+1, t}(k) s'_{+1, t}(k) - \left(s_t(k_{+1, t}) + \beta K \right) \geq \beta K - 1 - \beta K = -1 = a_t.
    \end{align*}
    Therefore, $c_t = b_t - a_t = \frac{\gamma + 1}{\gamma}$. Using this value of $c_t$ in \eqref{eq:azuma_martingale} and setting $\exp \left( -\frac{\epsilon^2}{2 \sum_{t=1}^T c_t^2} \right) \leq \frac{\delta}{2}$ gives $\epsilon = \frac{\gamma + 1}{\gamma} \sqrt{2T \ln \frac{2}{\delta}}$ finishing the proof.
\end{proof}

We now show a high probability regret bound on $\dbreg^+_T(j^T)$.

\begin{lemma}
    \label{lemma:3s_borda_row_player_bound}
    Let $j^T = (j_1, \dots, j_T)$ be an arbitrary sequence of $T$ actions with $S$ switches. The following bound holds with probability at least $1 - \delta$ when $\eta = \left( \frac{S\ln K}{T \sqrt{2K}} \right)^{2/3}$, $\beta = \frac{S^{1/3} \sqrt{\ln (2K/\delta)}}{(2\eta)^{1/4} K^{3/4} \sqrt{T}}$, $\gamma = \sqrt{2 \eta K} \geq \sqrt{\eta K (1 + \beta)}$, and $\alpha = \frac{1}{KT}$:
    \begin{align*}
        \dbreg^+_T(j^T) = \tilde{O}(S^{1/6} K^{-1/3} T^{5/6} + S^{1/2} K^{1/3} T^{2/3}).
    \end{align*}
\end{lemma}

\begin{proof}
Divide the time horizon into sub-intervals $[T_1, T_2), [T_2, T_3), \dots, [T_S, T_{S + 1})$ as in \cref{appendix:3s_analysis}. Using \cref{lemma:sum_psprime_high_prob_bound} and the condition $\frac{\eta K}{\gamma^2} (1 + \gamma \beta) \leq 1$, and performing a similar calculation as in \cref{appendix:3s_analysis} yields for any action $j^s$, with probability at least $1 - \delta/2$,
\begin{align*}
    -\ln \frac{1}{\alpha} + \eta \sum_{t = T_s}^{T_{s + 1} - 1} s'_{+1, t}(j^s) \leq \ln \frac{W_{+1, T_{s + 1}}}{W_{+1, T_s}} \leq &\frac{\eta}{1 - \gamma} \beta K \Delta^s + \frac{\eta}{1 - \gamma} \sum_{t = T_s}^{T_{s + 1} - 1} s_t(k_{+1, t}) + \frac{\eta (1 + \gamma)}{\gamma (1 - \gamma)} \sqrt{2 \Delta^s \ln\frac{2}{\delta}} \\
    &\frac{\eta^2 (1 + \gamma\beta)}{\gamma(1 - \gamma)} \sum_{t = T_s}^{T_{s + 1} - 1} \sum_{k=1}^K s'_{+1, t}(k) + e\alpha K \Delta^s,
\end{align*}
where $\Delta^s = T_{s + 1} - T_s$ is the length of the $s^{th}$ sub-interval. Simplifying as in \cref{appendix:3s_analysis}, we get with probability at least $1 - \delta/2$,
\begin{align*}
    \sum_{t = T_s}^{T_{s + 1} - 1} s_t(k_{+1, t}) + \frac{1 + \gamma}{\gamma} \sqrt{2\Delta^s \ln (2/\delta)} \geq (1 - 2\gamma) \max_{k \in [K]} \sum_{t=T_s}^{T_{s+1}-1} s'_{+1, t}(k) - \left( \frac{e\alpha}{\eta} + \beta\right)K \Delta^s - \frac{1}{\eta} \ln \frac{1}{\alpha}.
\end{align*}
Using \cref{lemma:s_bound_borda_score}, we get with probability at least $1 - \delta$
\begin{align*}
    \sum_{t = T_s}^{T_{s + 1} - 1} s_t(k_{+1, t}) + \frac{1 + \gamma}{\gamma} \sqrt{2\Delta^s \ln (2/\delta)} &\geq (1 - 2\gamma) \max_{k \in [K]} \sum_{t=T_s}^{T_{s+1}-1} s_t(k) - \left( \frac{e\alpha}{\eta} + \beta\right)K \Delta^s - \frac{1}{\eta} \ln \frac{1}{\alpha} - \frac{1}{\gamma\beta} \ln \frac{2K}{\delta} \\
    &\geq (1 - 2\gamma) \sum_{t=T_s}^{T_{s+1}-1} s_t(j^s) - \left( \frac{e\alpha}{\eta} + \beta\right)K \Delta^s - \frac{1}{\eta} \ln \frac{1}{\alpha} - \frac{1}{\gamma\beta} \ln \frac{2K}{\delta}.
\end{align*}
Here, recall that $j^s$ is the constant arm in the sequence $j^T = (j_1, j_2, \dots, j_T)$ in the sub-interval $[T_s, T_{s + 1})$. Noting that $s_t(k) \leq 1$ for any arm $k$, summing over $s = 1, \dots, S$, and rearranging, we get,
\begin{align*}
    \sum_{t=1}^T \left( s_t(j_t) - s_t(k_{+1, t}) \right) \leq 2 \gamma T + \frac{1 + \gamma}{\gamma} \sqrt{2 S T \ln (2/\delta)} + \left( \frac{e\alpha}{\eta} + \beta\right)K T + \frac{S}{\eta} \ln \frac{1}{\alpha} + \frac{S}{\gamma\beta} \ln \frac{2K}{\delta}.
\end{align*}
Setting $\eta = \left( \frac{S\ln K}{T \sqrt{2K}} \right)^{2/3}$, $\beta = \frac{S^{1/3} \sqrt{\ln (2K/\delta)}}{(2\eta)^{1/4} K^{3/4} \sqrt{T}}$, $\gamma = \sqrt{2 \eta K} \geq \sqrt{\eta K (1 + \beta)}$, and $\alpha = \frac{1}{KT}$ finishes the proof.

\end{proof}

\end{document}